\pdfoutput=1

\documentclass[a4paper,fleqn]{cas-sc}

\usepackage[authoryear]{natbib}

\usepackage{amsthm}
\usepackage{listings}
\usepackage{mathtools}
\usepackage{microtype}
\usepackage{placeins}
\usepackage{xcolor}

\definecolor{minoCodeBackground}{RGB}{247,249,252}
\definecolor{minoCodeFrame}{RGB}{183,192,204}
\definecolor{minoCodeKeyword}{RGB}{21,78,121}
\definecolor{minoCodeLibrary}{RGB}{111,66,193}
\definecolor{minoCodeComment}{RGB}{44,122,74}
\definecolor{minoCodeString}{RGB}{163,21,55}
\definecolor{minoCodeNumber}{RGB}{98,108,123}
\definecolor{minoCodeText}{RGB}{31,35,40}

\lstdefinestyle{mino-python}{
  language=Python,
  basicstyle=\ttfamily\footnotesize,
  backgroundcolor=\color{minoCodeBackground},
  identifierstyle=\color{minoCodeText},
  keywordstyle=\color{minoCodeKeyword}\bfseries,
  commentstyle=\color{minoCodeComment}\itshape,
  stringstyle=\color{minoCodeString},
  emph={jax,jnp},
  emphstyle=\color{minoCodeLibrary}\bfseries,
  columns=fullflexible,
  keepspaces=true,
  numbers=left,
  numberstyle=\ttfamily\tiny\color{minoCodeNumber},
  numbersep=7pt,
  stepnumber=1,
  numberblanklines=false,
  showspaces=false,
  showstringspaces=false,
  showtabs=false,
  tabsize=4,
  breaklines=true,
  breakatwhitespace=false,
  frame=single,
  rulecolor=\color{minoCodeFrame},
  framerule=0.4pt,
  framesep=4pt,
  xleftmargin=2.3em,
  framexleftmargin=1.8em,
  xrightmargin=0.5em,
  framexrightmargin=0.2em,
  aboveskip=0.8em,
  belowskip=0.8em,
  abovecaptionskip=0.2em,
  belowcaptionskip=0.5em,
  captionpos=t
}

\theoremstyle{plain}
\newtheorem{theorem}{Theorem}[section]
\newtheorem{proposition}[theorem]{Proposition}
\newtheorem{lemma}[theorem]{Lemma}
\newtheorem{corollary}[theorem]{Corollary}
\theoremstyle{definition}
\newtheorem{definition}[theorem]{Definition}
\newtheorem{example}[theorem]{Example}
\newtheorem{remark}[theorem]{Remark}

\newcommand{\R}{\mathbb{R}}

\newcommand{\Tst}{T^{*}\Omega}
\newcommand{\wh}[1]{\widehat{#1}}
\newcommand{\norm}[1]{\lVert #1 \rVert}
\newcommand{\abs}[1]{\lvert #1 \rvert}

\newcommand{\dd}{\,\mathrm{d}}
\newcommand{\WF}{\mathrm{WF}}

\renewcommand{\eqref}[1]{Eq.~(\ref{#1})}
\newcommand{\eqrefs}[2]{Eqs.~(\ref{#1}) and~(\ref{#2})}

\newcommand{\ie}{\emph{i.e.}}
\newcommand{\eg}{\emph{e.g.}}
\ExplSyntaxOn
\RenewDocumentCommand \printorcid { }
{
  \seq_if_empty:NF \g_stm_orcid_seq
    {
      \group_begin:
        \tex_let:D \thefootnote \relax \footnotetext
        {
          \raggedright
          \textsc{orcid}(s):\c_space_token
          \seq_use:Nn \g_stm_orcid_seq { ;~ }
        }
      \group_end:
    }
}
\cs_set:Npn \__first_footerline:
{
  \group_begin:
    \small
    \sffamily
    \__short_authors: :~
    { \rmfamily \itshape Preprint }
  \group_end:
}
\ExplSyntaxOff

\makeatletter
\AtEndDocument{%
  \clearpage
  \immediate\write\@auxout{%
    \string\csxdef{lastpage}{\the\numexpr\value{page}-1\relax}}}
\makeatother

\begin{document}
\let\WriteBookmarks\relax
\def\floatpagepagefraction{1}
\def\textpagefraction{.001}

\shorttitle{Microlocal neural operators}
\shortauthors{G.L.R. N'guessan and B.J. Kim}

\title[mode = title]{MiNO: Cotangent-bundle propagator learning for PDEs}

\author[1,2,3]{Gnankan Landry Regis N'guessan}[orcid=0009-0009-6161-4608]
\ead{rnguessan@aimsric.org}
\credit{Conceptualization, Methodology, Software, Formal analysis, Investigation, Data curation, Validation, Visualization, Writing -- original draft, Writing -- review \& editing}

\author[4]{Bum Jun Kim}[orcid=0000-0003-4155-9225]
\cormark[1]
\ead{bumjun.kim@weblab.t.u-tokyo.ac.jp}
\credit{Formal analysis, Visualization, Supervision, Validation, Writing -- review \& editing}

\affiliation[1]{op={},organization={Axiom Research Group}}
\affiliation[2]{organization={Department of Applied Mathematics and Computational Science, The Nelson Mandela African Institution of Science and Technology}, addressline={404 Nganana, Kikwe, Arumeru, P.O. Box 447}, city={Arusha}, postcode={23311}, country={Tanzania}}
\affiliation[3]{organization={African Institute for Mathematical Sciences, Research and Innovation Centre}, addressline={KN 3 Road, Gasharu Cell, Kicukiro Sector, Kicukiro District}, city={Kigali}, country={Rwanda}}
\affiliation[4]{organization={Graduate School of Engineering, The University of Tokyo}, addressline={7-3-1 Hongo, Bunkyo-ku}, city={Tokyo}, postcode={113-8656}, country={Japan}}

\cortext[cor1]{Corresponding author}

\begin{abstract}
Scientific machine learning for partial differential equations commonly targets solution fields, as in physics-informed neural networks, or solution maps, as in neural operators. We study a third target: the propagator itself, a phase and amplitude in phase space. The motivation is a gap in regularity. A transported discontinuity is nonsmooth in space and time, yet the rule that moves it can be a polynomial phase carrying unit amplitude, so the object that generates an evolution can be far smoother than the field it generates. The microlocal neural operator (MiNO) learns that object, using the eikonal equation for the phase and the transport equation for the amplitude, and recovers the solution by an oscillatory integral. Sharp fronts and caustics then belong to propagation geometry rather than to a field fitted pointwise. Small residuals certify more than the reconstructed field. They place the learned canonical relation, the geometry that carries singularities, close to the exact one, and they separate trainable error from the frequency-truncation tail. On a matched-budget discontinuous-advection benchmark, MiNO stops improving within 10,000 steps at the accuracy limit of its finite reconstruction window, a limit predicted in closed form, whereas a physics-informed neural network with neural-tangent-kernel loss balancing stays near its initial error. On smooth advection, the mean error is $3.84\times10^{-3}$ for MiNO and $3.12\times10^{-2}$ for a supervised Fourier neural operator. Single-branch MiNO is the smallest model compared, and one trained generator serves five unseen initial conditions without retraining.
\end{abstract}


\maketitle
\hypersetup{
  pdfauthor={Gnankan Landry Regis N'guessan and Bum Jun Kim},
  pdfsubject={Propagator learning for partial differential equations},
  pdfkeywords={Propagator learning; Microlocal analysis;
    Fourier integral operator; Eikonal equation; Neural operator;
    Physics-informed learning},
}

\section{Introduction}
\label{sec:intro}

Most neural solvers for partial differential equations (PDEs) learn the solution itself. Physics-informed neural networks (PINNs) represent a field over space--time and enforce the governing equation through residual losses \citep{raissi2019pinn,sirignano2018dgm,e2018deepritz}, whereas neural operators learn maps between function spaces \citep{li2021fno,lu2021deeponet,kovachki2023neural}. These approaches have enabled broad progress in scientific machine learning \citep{karniadakis2021piml}. This paper adopts a different target. Rather than learning only the visible solution, we learn the phase-space mechanism that propagates that solution.

We formulate this target through the phase--amplitude pair $(\phi,A)$ of a Fourier-integral representation of the propagator. The phase describes how oscillations and fronts travel, the amplitude describes their strength, and an oscillatory reconstruction returns the solution. We call this paradigm propagator learning and instantiate the paradigm through the microlocal neural operator (MiNO). MiNO learns the phase and amplitude on the cotangent bundle by enforcing their eikonal and transport equations.

The motivation is clearest for sharp propagation. A moving discontinuity is nonsmooth as a function of space and time, yet the rule that moves the discontinuity can be smooth and simple in phase space. For unit-speed advection, for example, a transported square wave lies in $L^2(\R)$ but in no $H^s(\R)$ with $s\ge1/2$, while its propagator has the polynomial phase $x\xi-t\xi$ and unit amplitude. The field and its generator therefore describe the same evolution at very different levels of regularity.

That gap is not only a matter of training difficulty. The solution manifold of a convection-dominated problem has a slowly decaying Kolmogorov $n$-width, and, for linear transport and wave problems, that width obeys an algebraic lower bound \citep{greif2019nwidth,ohlberger2016reduced}, so no fixed low-dimensional linear approximation of the field is uniformly accurate in time. The established remedies for sharp transport act on the approximation frame or on the optimization while leaving the learned object a field over space--time (Appendix~\ref{app:related}). Changing the target moves the problem to an object that the lower bound does not describe.

This phase-space view is classical in high-frequency propagation, where fronts, caustics, and multivalued arrivals are organized by bicharacteristic flow \citep{engquist2003high,osher2002geometric,hormander1971fio, duistermaat1996fio,hormander1985vol3,maslov1981semiclassical}. Learning the phase makes that motion explicit through the propagator's canonical relation, the rule that says where a singularity sitting at one point of phase space has arrived by a later time. A pseudo-differential symbol instead keeps the fixed identity phase and hence the diagonal relation: Such an operator can sharpen or damp a singularity but cannot move one, whatever symbol is learned (Remark~\ref{rem:pseudo}).

The experiments reflect this difference in representation. On a matched-budget discontinuous-advection test, MiNO stops improving within 10,000 optimization steps because the model has reached the accuracy limit of its finite reconstruction window, a limit computed in closed form in Proposition~\ref{prop:square-tail}, while a PINN with neural-tangent-kernel (NTK) loss balancing remains near its initial error (Sec.~\ref{sec:exp-discont}).

Both methods receive the same governing equation and are asked for different objects; Sec.~\ref{sec:discussion} reads the reported margins in that light. The separation also has a structural source. MiNO satisfies the eikonal Cauchy datum by construction, so the initial-condition penalty that a PINN must balance against its equation residual is absent from the objective (Secs.~\ref{sec:mino-param} and~\ref{sec:exp-discont}).

The learned generator is independent of the initial datum: New data enter only through their Fourier transforms during reconstruction. One trained generator is therefore reused for five distinct initial conditions without retraining (Sec.~\ref{sec:exp-multiic}), a reuse that follows from the linearity of the representation rather than from the coverage of a training set. The same structure supports a bank of forward generators for inverse parameter recovery (Sec.~\ref{sec:exp-inverse}), and the learned generator returns the canonical relation itself, so the propagated position of a singularity is read from the learned phase rather than inferred from a reconstructed field. On smooth advection, MiNO also attains a lower mean error than a supervised Fourier neural operator (FNO) (Sec.~\ref{sec:exp-fno}). These uses require the governing symbol but not labeled solution trajectories.

Related methods use phase-space structure for different targets. Learned eikonal solvers predict a base-space traveltime, learned-symbol operators retain a pseudo-differential phase, and imaging networks inspired by Fourier integral operators (FIOs) learn an imaging map. MiNO instead learns the phase and amplitude of a PDE propagator and reconstructs the evolving field, including its branch structure. Appendix~\ref{app:related} develops this comparison in detail.

Figure~\ref{fig:graphical-abstract} summarizes the representational shift, the phase-space learning mechanism, and the matched-budget sharp-front result.

\begin{figure}[pos=t!]
\centering
\includegraphics[width=\linewidth]{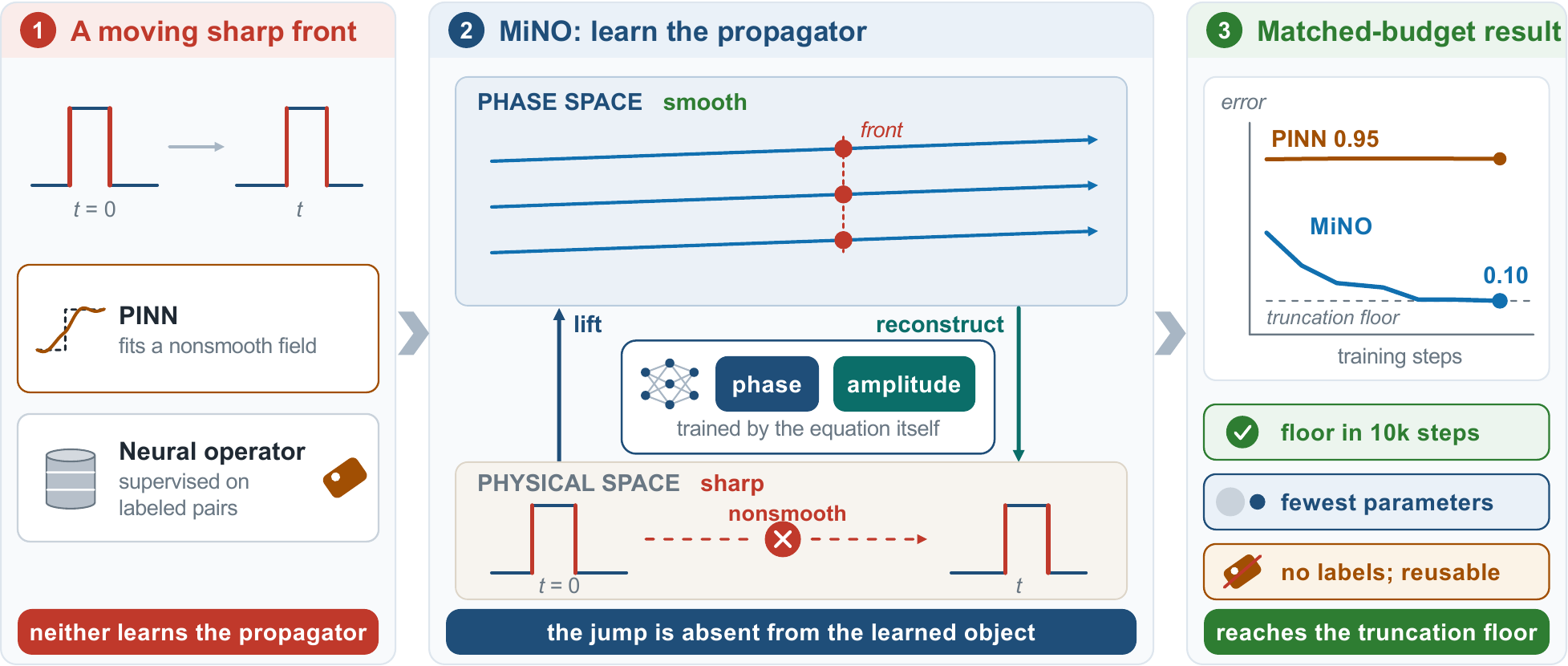}
\caption{Graphical overview of the MiNO paradigm. Conventional field learning must fit a moving discontinuity in physical space, while supervised operator learning is trained on labeled pairs. MiNO instead learns smooth phase--amplitude propagator data in phase space and reconstructs the sharp front. The chart at the right sketches the matched-budget discontinuous-advection benchmark of Sec.~\ref{sec:exp-discont}, in which MiNO, using no labeled trajectories, settles near the analytic frequency-truncation tail scale while the NTK-balanced PINN stays near its initial error. The plotted MiNO curve and PINN level follow the measured means of Table~\ref{tab:learncurve} on a logarithmic error axis, drawn without numerical axes; that PINN has converged to the trivial solution $u\equiv0$ (Sec.~\ref{sec:exp-discont}).}
\label{fig:graphical-abstract}
\end{figure}

\paragraph{Contributions.}
\begin{enumerate}
\item A propagator-learning formulation and architecture. We make the propagator the learned object and introduce MiNO, whose time-factored parameterization satisfies the eikonal Cauchy datum and the amplitude normalization by construction, so neither enters the loss (Secs.~\ref{sec:targets} and~\ref{sec:mino}). Because the learned phase carries a nondiagonal canonical relation, the target admits transport of singularities, which no learned pseudo-differential symbol can represent (Remark~\ref{rem:pseudo}).

\item Certification of the learned geometry. Differentiated eikonal residuals bound the distance between the learned canonical relation and the exact one (Theorem~\ref{thm:canonical-residual} and Corollary~\ref{cor:canonical-uniform}). The certified object is the propagation geometry rather than a function, and that object has no counterpart in field or solution-map error analysis. The same estimates identify which residual norm certifies that geometry: The scalar residual alone does not (Remark~\ref{rem:residual-norm}).

\item Residual-to-error analysis of the reconstruction. We relate continuous phase and amplitude residuals to phase error, reconstruction error, and scalar-transport propagator error. A full-solution estimate then separates generator error from the frequency-truncation tail, including a closed-form tail for discontinuous data (Sec.~\ref{sec:theory}). That tail is $0.1028$ on the benchmark window, against a measured late-budget error of $0.103$, so the analysis predicts the level at which training stops paying.

\item Evaluation across propagation and reuse settings. Twelve studies show that the target reaches the accuracy floor its analysis predicts while a matched field baseline does not, that one generator serves five unseen initial conditions and a bank of generators drives inverse speed recovery, and that these results are obtained without labeled trajectories. The same construction carries to bidirectional waves, two variable-coefficient symbols, a point focus, and four-dimensional phase space (Sec.~\ref{sec:experiments}).
\end{enumerate}

Taken together, these contributions develop MiNO for linear Cauchy propagators through chartwise stability theory and one- and two-dimensional propagation experiments, and the contributions show that the difficulty of a neural PDE problem is set as much by the object one chooses to learn as by the equation one is given.

\section{Representational targets and propagator formulation}
\label{sec:targets}

We define the three targets and record the bicharacteristic kernel relation of the localized homogeneous FIO representation; Appendix~\ref{app:related} places each target against the corresponding literature. Figure~\ref{fig:schematic} summarizes the taxonomy.

\begin{figure}[pos=t!]
\centering
\includegraphics[width=\linewidth]{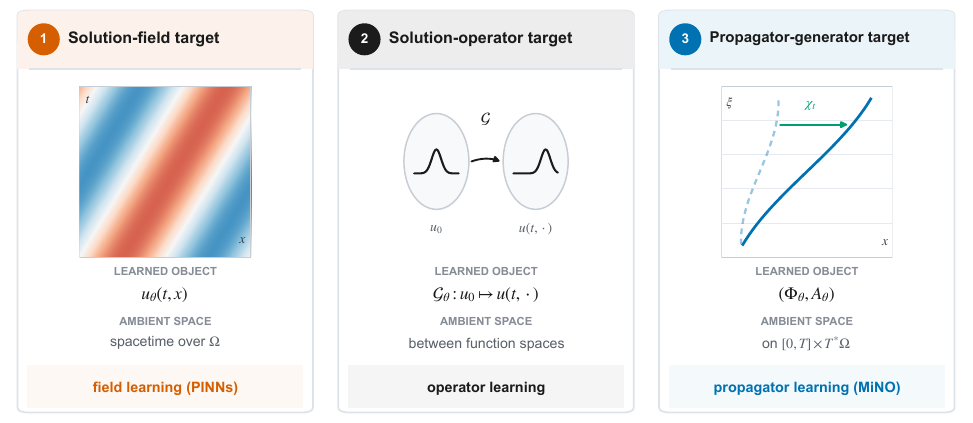}
\caption{The three representational targets. Field learning targets a solution field over space--time; operator learning targets a map between function spaces; and propagator learning targets chartwise phase--amplitude data on the time-parameterized cotangent bundle. In the admissible homogeneous hyperbolic setting of Proposition~\ref{prop:causal}, the phase generates the bicharacteristic canonical relation.}
\label{fig:schematic}
\end{figure}

\subsection{Solution-field learning}
\label{sec:targets-field}
PINNs learn $u_{\theta}(t,x)$. Physics enters through residuals, initial and boundary terms, causal or adaptive weights, or NTK balancing \citep{raissi2019pinn,wang2021gradient,wang2022causal,wang2022ntk, krishnapriyan2021failure}. The learned object is a field on the base manifold. Field learning asks whether a network can approximate the visible solution while being penalized for violating the equation. The penalty does not change the domain of the learned object. The field remains a function on $[0,T]\times\Omega$. Two difficulties of that domain are documented independently of any particular optimizer. Coordinate networks are biased toward low frequencies \citep{rahaman2019spectral,xu2020frequency} unless their inputs are lifted by Fourier features \citep{tancik2020fourier}, and the solution manifold of a convection-dominated problem has a slowly decaying Kolmogorov $n$-width, so no fixed low-dimensional linear subspace approximates that manifold uniformly in time \citep{mojgani2023lpinn}. The same slow $n$-width decay limits linear model reduction for transport-dominated problems \citep{ohlberger2016reduced}, and, for linear transport and wave problems, that width admits an explicit algebraic lower bound \citep{greif2019nwidth}.

\subsection{Solution-operator learning}
\label{sec:targets-op}
Neural operators learn $\mathcal{G}_{\theta}:u_{0}\mapsto u(t,\cdot)$ \citep{li2021fno,lu2021deeponet,kovachki2023neural}. Operator learning asks whether a network can amortize the input--output map of a PDE family from data. The learned object is a point in a space of maps between function spaces. In the supervised setting considered here, that map is trained on input--output pairs and, in widely used architectures, is represented by a composition of integral kernels and pointwise nonlinearities. The propagation geometry of the underlying equation is not explicitly parameterized; that geometry is encoded implicitly through the learned input--output map and its training distribution.

\subsection{Propagator learning}
\label{sec:targets-gen}
\paragraph{Problem setting and symbol conventions.} We make the third target precise. We consider linear Cauchy propagators on $\Omega=\R^{d}$, on $\Omega=\mathbb T^{d}$ with Fourier series, or in a localized coordinate chart with compact cutoffs. In Euclidean charts, we use the Fourier representation below. Consider the linear evolution problem
\begin{equation}
\label{eq:cauchy}
(D_t+p(x,D_x))u=0,\qquad u(0,\cdot)=u_0, \qquad D_t\coloneqq-i\partial_t,\quad D_x\coloneqq-i\nabla_x,
\end{equation}
where the $D=-i\partial$ convention fixes the signs in the eikonal equation and Hamiltonian flow. The notation $p(x,D_x)$ also requires a quantization convention. The scalar transport examples below use the left-quantized operator $c(x)D_x$, whereas the leading half-density formula uses the Weyl convention and makes its subprincipal symbol $p_{\mathrm{sub}}$ explicit \citep{zworski2012semiclassical}. Thus two operators with the same principal symbol need not have the same amplitude equation. We distinguish two settings. In the hyperbolic FIO setting, $p$ is real and positively homogeneous of degree one in $\xi$ away from the zero section. In the quadratic (metaplectic) setting, $p(x,\xi)=w^THw/2+\ell\cdot w+c_0$, with $w\coloneqq(x,\xi)$ and $H=H^T$ real, is a real quadratic Hamiltonian. Appendix~\ref{app:transport} treats the translation-invariant quadratic-multiplier subclass $p=p(\xi)$. Write $\wh{v}(\xi)\coloneqq(2\pi)^{-d/2}\int e^{-ix\cdot\xi}v(x)\dd x$ for the Euclidean Fourier transform. Appendix~\ref{app:notation} collects the symbols that recur throughout the paper.

\paragraph{Fourier-integral propagator representation.} For the third target, we recall the local structure of the propagator. In the homogeneous hyperbolic case, the propagator is locally an FIO \citep{hormander1971fio,duistermaat1996fio,hormander1985vol4}. Its geometric-optics prototype, an eikonal phase carrying an amplitude transported along rays, was constructed for oscillatory hyperbolic Cauchy data by \citet{lax1957asymptotic}. A global propagator generally requires a locally finite sum of phase charts, cutoffs, branch amplitudes, and Maslov transition factors. The transition factors are governed by the characteristic class on the Lagrangian Grassmannian introduced by \citet{arnold1967maslov}. Homogeneous quadratic Hamiltonians generate metaplectic operators; affine and constant terms add the corresponding Heisenberg translation and scalar phase. Metaplectic operators have a separate global phase-space $L^2$ theory \citep{folland1989phasespace,zworski2012semiclassical}, while the local wavefront result below applies to the homogeneous hyperbolic setting. On one fixed chart and frequency window, which is the form used in the experiments, a branch is written
\begin{equation}
\label{eq:fio}
u(t,x) = (F_{\phi,A}u_{0})(x) \coloneqq (2\pi)^{-d/2} \int_{\R^{d}} e^{i\phi(t,x,\xi)}A(t,x,\xi) \wh u_0(\xi)\dd\xi ,
\end{equation}
where compact chart cutoffs are included in $A$. On $\mathbb T^d$, the integral is replaced by the corresponding Fourier series. Globally, one sums expressions of the form in \eqref{eq:fio} over admissible branches; a branch weight may be either absorbed into its amplitude or factored out as an explicit external coefficient. Accordingly, the normalization $A(0)=1$ used below is the normalization of one dynamical branch before any explicit external branch weight, while the total branch sum must reproduce the prescribed Cauchy datum. In a real hyperbolic chart, $\phi$ is positively homogeneous of degree one in $\xi$ and solves the eikonal equation
\begin{equation}
\label{eq:eikonal}
\partial_{t}\phi + p(x,\nabla_{x}\phi) = 0, \qquad \phi(0,x,\xi) = x\cdot\xi ,
\end{equation}
and $A$ solves the associated chart transport equation along the Hamiltonian flow of $p$. The pair $(\phi,A)$ is a chart representation on the cotangent bundle $\Tst=\{(x,\xi)\}$; branch and Maslov data are part of the global generator.

\begin{definition}[Phase--amplitude generator target]
\label{def:gen}
The phase--amplitude generator target is the equivalence class of chartwise phase--amplitude data $(\phi,A)$ defined on $[0,T]\times\Tst$ by \eqrefs{eq:fio}{eq:eikonal}, together with the branch and Maslov transition data needed globally. These data generate the evolution through the oscillatory representation in \eqref{eq:fio}. We use the term generator for these representation data; $-ip(x,D_x)$ remains the infinitesimal semigroup generator of the evolution in \eqref{eq:cauchy}.
\end{definition}

\subsection{Comparison of the three targets}
\label{sec:targets-comparison}
The three targets occupy different ambient spaces: a field over space--time, a map between function spaces, and chartwise phase--amplitude data on the time-parameterized cotangent bundle. Table~\ref{tab:targets} sets the three side by side. The third makes the canonical geometry of the equation explicit through \eqref{eq:eikonal}.

\begin{table}[pos=t!]
\centering
\caption{Attributes of the three paradigms: learned object, ambient space, training signal, and output.}
\label{tab:targets}
\small
\begin{tabular*}{\linewidth}{@{\extracolsep{\fill}}lllll@{}}
\toprule
Paradigm & Learned object & Space & Training signal & Output \\
\midrule
PINN & $u_{\theta}(t,x)$ & base manifold & PDE residual & solution field \\
Neural operator & $\mathcal{G}_{\theta}$ & function-space map & supervised pairs & solution map \\
MiNO & $(\Phi_{\theta},A_{\theta})$ & cotangent bundle & eikonal and transport & propagator reconstruction \\
\bottomrule
\end{tabular*}
\end{table}

\subsection{Bicharacteristic kernel relation and locality}
\label{sec:targets-kernel}
The following proposition identifies the kernel relation of the localized homogeneous hyperbolic FIO representation, for an admissible exact phase. Its proof and the proofs of the results of Sec.~\ref{sec:theory} are collected in Appendix~\ref{app:proofdetails}.

\begin{proposition}[Bicharacteristic kernel relation]
\label{prop:causal}
Let $p\in C^\infty(T^*\Omega\setminus0;\R)$ be positively homogeneous of degree one in $\xi$, and let $\phi$ be a real, degree-one homogeneous admissible phase solving \eqref{eq:eikonal} in a localized hyperbolic FIO chart. For a fixed time $t$ in that chart, let $A(t,x,\xi)$ be a smooth classical amplitude symbol supported in a conic frequency region away from the zero section. Suppose that $\partial^2_{x\xi}\phi$ is nonsingular there and that the resulting localized operator $F_{\phi,A}(t)$ is properly supported. Let $\mathcal K_t$ be its Schwartz kernel. Then
\begin{equation}
\WF'(\mathcal K_t) \subseteq \{(x,\xi_{x};y,\xi_{y}) : (x,\xi_{x}) = \chi_{t}(y,\xi_{y})\},
\end{equation}
where $\chi_t$ is the time-$t$ Hamiltonian flow of $p$.
\end{proposition}

\begin{remark}[Locality of the graph condition]
\label{rem:local}
The mixed-Hessian condition identifies a local graph for the propagator canonical relation $\mathcal C_t$ in the chosen phase chart. Caustics are degeneracies of the base projection of a propagated solution Lagrangian. For a chosen oscillatory initial datum with initial Lagrangian $\Lambda_{S_0}$, write $\Lambda_t^{S_0}\coloneqq\mathcal C_t\circ\Lambda_{S_0}\subset T^*\Omega$. A caustic is a degeneracy of $\pi_x:\Lambda_t^{S_0}\to\Omega$, or equivalently of a specified stationary branch of the reconstruction. In an adapted stationary-phase chart, a Hessian determinant represents its projection Jacobian up to a nonvanishing factor; this representation depends on the chosen chart. The two conditions can fail independently. In the Airy generating-family normal form $x\xi-\xi^3/3$, $\partial^2_{x\xi}\phi=1$ at the fold while $\partial^2_{\xi\xi}\phi=0$. This local normal form separates the two geometric nondegeneracy conditions and is distinct from a degree-one homogeneous propagator phase. Proposition~\ref{prop:causal} treats the local operator chart; Sec.~\ref{sec:theory-caustic} treats the projection of a specified solution Lagrangian.
\end{remark}

The same proposition separates the present target from its closest learning-based neighbor, in which a network parameterizes a pseudo-differential symbol and the solution operator remains the target \citep{shin2024pdno}.

\begin{remark}[A learned symbol cannot move a wavefront set]
\label{rem:pseudo}
A pseudo-differential operator $P$ carries the fixed identity phase $(x-y)\cdot\xi$, so its canonical relation is the diagonal, and $P$ is microlocal: $\WF(Pu)\subseteq\WF(u)$ for every symbol \citep{hormander1985vol3}. The restriction is structural rather than a matter of capacity. No choice of learned symbol moves a wavefront set, so this family cannot represent the transport of singularities at all, however the symbol network is trained. In Proposition~\ref{prop:causal} the propagator relation is instead the graph of $\chi_t$, and transport of singularities is exactly its nondiagonality. Learning the phase is therefore what admits the phenomenon; the amplitude then supplies the frequency-dependent strength along the same geometry.
\end{remark}

\section{Microlocal neural operators}
\label{sec:mino}

MiNO instantiates propagator learning. The single design principle is the time-factored parameterization of the phase increment and chart amplitude, so that the eikonal Cauchy datum and amplitude normalization are imposed by construction rather than by penalty. The construction comprises the symbol and its Hamiltonian flow, the training residuals, the phase and amplitude parameterizations, the reconstruction quadrature, and the symbol-dependent procedure that turns a written symbol into the training problem.

\subsection{Principal symbol and Hamiltonian flow}
\label{sec:mino-flow}
The construction starts from the principal symbol $p(x,\xi)$ of \eqref{eq:cauchy}. Its Hamiltonian vector field on $\Tst$ is
\begin{equation}
\label{eq:hamflow}
\dot{x} = \nabla_{\xi}p(x,\xi), \qquad \dot{\xi} = -\nabla_{x}p(x,\xi),
\end{equation}
and the bicharacteristics are the integral curves of \eqref{eq:hamflow} lifted to the characteristic hypersurface $\{(t,x,\tau,\xi):\allowbreak\tau+p(x,\xi)=0\}$. Three consequences of \eqref{eq:hamflow} are relevant to the subsequent analysis. First, the projections of the bicharacteristics to $\Omega$ are the rays along which the eikonal equation in \eqref{eq:eikonal} is solved by the method of characteristics; this correspondence is the identification used in the proof of Proposition~\ref{prop:causal}. Second, the $\xi$-component of \eqref{eq:hamflow} is driven by $\nabla_{x}p$, so for a spatially homogeneous symbol $p=p_{0}(\xi)$, the frequency is conserved along rays and the phase is exactly $x\cdot\xi-tp_{0}(\xi)$; under the parameterization in \eqref{eq:phase}, the exact phase-increment target is therefore $h^*(t,x,\xi)=-p_0(\xi)$. This spatially homogeneous case is the calibration regime tested in Sec.~\ref{sec:exp-smooth}. Third, for a specified ray-parameterized solution Lagrangian $x=x(t;y)$, the relevant projection Jacobian is
\begin{align*}
J_{\mathrm{proj}}(t,y)\coloneqq\det\frac{\partial x(t;y)}{\partial y}.
\end{align*}
In an adapted stationary-phase normal form, a phase Hessian represents this Jacobian up to a nonvanishing factor at the critical point. Remark~\ref{rem:local} records the chart dependence of that representation, and Sec.~\ref{sec:theory-caustic} illustrates the normal-form geometry. The method derives directly from the propagation geometry. The curves along which the phase is constrained are the bicharacteristics of $p$.

\subsection{The training principle}
\label{sec:mino-train}
\paragraph{Eikonal and transport residuals.} Let $\Phi_{\theta}$ and $A_{\theta}$ be parameterized families on $[0,T]\times\Tst$. Here $\xi$ is the initial Fourier (fiber) label of the mixed generating function, not an independently evolving current momentum; the current covector is $\zeta\coloneqq\nabla_x\Phi_\theta(t,x,\xi)$. Define
\begin{equation}
\label{eq:Reik}
\mathcal{R}_{\mathrm{eik}}[\Phi_{\theta}](t,x,\xi) \coloneqq \partial_{t}\Phi_{\theta} + p(x,\nabla_{x}\Phi_{\theta}), \qquad v_\theta(t,x,\xi) \coloneqq\nabla_\zeta p(x,\nabla_x\Phi_\theta(t,x,\xi)).
\end{equation}
The transport residual is chosen consistently with the target PDE, in a scalar or a Wentzel--Kramers--Brillouin (WKB) half-density form:
\begin{equation}
\label{eq:Rtr}
\mathcal R_{\mathrm{tr}}\coloneqq
\begin{cases}
\partial_tA_\theta+c(x)\partial_xA_\theta, & \text{scalar advection},\\
\partial_tA_\theta+v_\theta\cdot\nabla_xA_\theta +\tfrac12\operatorname{div}_x(v_\theta)A_\theta +ip_{\mathrm{sub}}(x,\nabla_x\Phi_\theta)A_\theta, & \text{half-density WKB, Weyl convention}.
\end{cases}
\end{equation}
In the second line, the target PDE supplies the quantization convention and the corresponding subprincipal symbol $p_{\mathrm{sub}}$; the quadratic Hamiltonians analyzed here have $p_{\mathrm{sub}}=0$. The divergence differentiates $x\mapsto\nabla_\zeta p(x,\nabla_x\Phi_\theta)$ at fixed fiber label $\xi$; equivalently, its geometric term is
\begin{align*}
\tfrac12\operatorname{tr}\left( p_{x\zeta}(x,\nabla_x\Phi_\theta) +p_{\zeta\zeta}(x,\nabla_x\Phi_\theta)\nabla_x^2\Phi_\theta \right)A_\theta.
\end{align*}
Thus for $p(\zeta)=\zeta^TH\zeta/2+\ell\cdot\zeta+c_0$ with $H=H^T$, the geometric term is $\operatorname{tr}(H\nabla_x^2\Phi_\theta)A_\theta/2$; that term reduces to $\Delta_x\Phi_\theta A_\theta/2$ when $H$ is the identity. The exact unit-scale Schr\"odinger amplitude equation for its Weyl quantization with $p_{\mathrm{sub}}=0$ adds $i\operatorname{tr}(H\nabla_x^2A_\theta)/2$ on the right-hand side, an $O(\varepsilon)$ correction in semiclassical scaling. Additional lower-order symbols enter according to their chosen quantization. The second line of \eqref{eq:Rtr} is therefore a leading WKB residual and is exact in quadratic charts with $x$-independent principal amplitude, including the multiplier case $\Phi=x\cdot\xi-tp(\xi)$, $A\equiv1$, and $p_{\mathrm{sub}}=0$. The scalar advection experiments use $\partial_tA+c\partial_xA=0$; adding $c'(x)A/2$ would instead target the half-density equation $\partial_tu+c\partial_xu+c'u/2=0$.

These formulas transport a chart amplitude restricted to the Lagrangian and expressed using the initial label $\xi$; a free function on $T^*\Omega$ follows the full Hamiltonian derivative. All trained experiments use the constant-amplitude specialization $A^*=1$ with $p_{\mathrm{sub}}=0$, for which the chart and experimental residuals have the same exact zero target. Appendix~\ref{app:transport} derives the symbol-specific reductions, and Appendix~\ref{app:experimental-residual} records the experimental specialization.

\paragraph{Training objective.} With $\Phi_\theta(0,x,\xi)=x\cdot\xi$ and, on the normalized initial identity chart before prescribed cutoffs or external branch weights, $A_\theta(0,x,\xi)=1$, the propagator-learning objective is
\begin{equation}
\label{eq:objective}
\mathcal{J}(\theta) \coloneqq \norm{\mathcal{R}_{\mathrm{eik}}[\Phi_{\theta}]}^{2}_{L^{2}(\mu)} + \lambda_{\mathrm{tr}} \norm{\mathcal{R}_{\mathrm{tr}}[\Phi_{\theta},A_{\theta}]}^{2}_{L^{2}(\mu)} + \lambda_{0}\mathcal{B}_{0}(\theta),
\end{equation}
with $\mu$ a sampling measure on $[0,T]\times\Tst$ and $\mathcal{B}_{0}$ an initial-condition term that is identically zero under the parameterization of Sec.~\ref{sec:mino-param}; the term also covers parameterizations that impose the Cauchy data through a penalty. The defining feature relative to solution learning is that \eqrefs{eq:Reik}{eq:Rtr} are equations on $\Tst$ for the generator, not equations on $\Omega$ for the field.

Theorem~\ref{thm:hj-phase-stability} converts a continuous phase residual into a compact-chart phase estimate for a general real Hamilton--Jacobi phase, and Corollary~\ref{cor:hj-reconstruction} combines that estimate with the time-slice amplitude error to control reconstruction. Theorem~\ref{thm:transport-residual} and Corollary~\ref{cor:transport-residual-l2} convert the continuous $L^2$ phase and amplitude residuals into an output-operator estimate for scalar transport. All four are stated for continuous residual norms.

Both ingredients of that conversion are available for the generator target. Tanh networks admit explicit simultaneous approximation bounds in higher-order Sobolev norms \citep{deryck2021tanh}, so a smooth chart generator is approximable in the norms these estimates use, and the estimates then return solution and canonical-geometry error. For a field target on a convection-dominated problem, the corresponding route through fixed low-dimensional linear approximation is constrained by the algebraic $n$-width lower bound recorded in Sec.~\ref{sec:targets-field}. That separation is a statement about approximation classes, not about the behavior of the optimizer.

\paragraph{Collocation sampling.} The sampling measure $\mu$ is a design variable in its own right. For residual collocation on the base manifold, a systematic comparison of nonadaptive and residual-based adaptive distributions shows that this choice materially affects accuracy \citep{wu2023sampling}. Here $\mu$ is the fixed product measure of Appendix~\ref{app:expdetails}, resampled each step using nonadaptive phase-space sampling. In experiments that enable online balancing, the residual weights are balanced by the NTK rule of \citet{wang2022ntk}; this balancing rule belongs to the optimizer, not to the representation.

\subsection{Phase and amplitude parameterizations}
\label{sec:mino-param}

\paragraph{Phase network.} The phase is parameterized in time-factored form,
\begin{equation}
\label{eq:phase}
\Phi_{\theta}(t,x,\xi)\coloneqq x\cdot\xi+th_{\theta}(t,x,\xi).
\end{equation}
The factor $t$ enforces $\Phi_{\theta}(0,\cdot,\cdot)=x\cdot\xi$ identically, so the eikonal Cauchy datum in \eqref{eq:eikonal} is satisfied exactly for every $\theta$ and never enters the loss; the initial-condition term $\mathcal{B}_{0}$ in \eqref{eq:objective} is then identically zero for the phase. The network directly parameterizes the full phase increment: $h_\theta=(\Phi_\theta-x\cdot\xi)/t$ represents the full time-averaged phase increment, whose smooth extension is understood at $t=0$. For a spatially homogeneous symbol, this increment is also the instantaneous phase rate, with exact target $h^*=-p_0(\xi)$.

On the finite window, the network parameterization is a numerical oscillatory ansatz. Degree-one homogeneity in $\xi$ can be imposed separately, \eg, as $h_\theta\coloneqq\abs{\xi}q_\theta(t,x,\xi/\abs{\xi})$ on conic charts, or holds in the exact zero-residual limit; the homogeneous wavefront theorem then applies.

The phase-increment network $h_{\theta}$ uses a modified multilayer perceptron (MLP) with multiplicative skip connections in the form stabilized for residual training \citep{wang2023expert} and sinusoidal Fourier features \citep{tancik2020fourier}; Appendix~\ref{app:expdetails} gives its full configuration. The network learns the geometry of propagation through the full phase increment. The features act on the generator, whose oscillation in $\xi$ is carried by the explicit exponential in \eqref{eq:fio} rather than by the network output.

\paragraph{Amplitude network.} Within the interior of each experimental chart, after any prescribed localization cutoff and explicit external branch weight have been factored out, the nonvanishing amplitude magnitude is parameterized multiplicatively, $B_{\theta}\coloneqq\exp(t\alpha_{\theta})$ with real $\alpha_\theta$, so $B_{\theta}(0,\cdot,\cdot)=1$ and positivity are enforced by construction. We write $A_\theta\coloneqq m_\nu B_\theta$, where $\abs{m_\nu}=1$ carries a fixed Maslov or chart-transition phase. On the initial identity chart, $m_\nu=1$, so $A_\theta(0)=1$. At a chart crossing, $m_\nu$ is applied as a chartwise transition factor; one such factor is $e^{-i\pi\nu/2}$, with $\nu$ the Maslov index of the crossing \citep{arnold1967maslov}. The total coefficient in \eqref{eq:fio} may vanish with its prescribed localization cutoff; the positivity property applies to $B_\theta$ in the chart interior. The present experiments use known $m_\nu$ on charts with zero subprincipal symbol.

The transport equation in \eqref{eq:Rtr} constrains $\alpha_{\theta}$ through its logarithmic derivative. Substituting $B_{\theta}=\exp(t\alpha_{\theta})$ and dividing by the nonzero $B_\theta$ gives a linear equation for $\alpha_\theta$ and its first derivatives, with the divergence term as an additive source. This algebraic normalization is the motivation for the parameterization. The amplitude network $\alpha_{\theta}$ uses the same modified-MLP family; Appendix~\ref{app:expdetails} gives its configuration. That network learns the smooth chart amplitude in phase space. Reduction of the uniform oscillatory integral to separate base-space branches subsequently produces the distinct branchwise stationary-phase coefficient.

The experiments parameterize deviations from their constant exact target as $A_\theta=1+t\gamma_\theta$, which enforces $A_\theta(0)=1$. The exponential model above additionally preserves positivity for general nonvanishing chart amplitudes; Appendix~\ref{app:experimental-residual} states the experimental specialization.

\subsection{Fourier-integral reconstruction}
\label{sec:mino-recon}
Given $(\Phi_{\theta},A_{\theta})$, the solution is recovered by \eqref{eq:fio}. MiNO outputs the generator, from which the solution is synthesized by the oscillatory integral. On a bounded frequency window, write the nonexponential factor as $g(t,x,\xi)\coloneqq A_{\theta}(t,x,\xi)\wh u_0(\xi)$ and the oscillatory factor as $e^{i\Phi_{\theta}(t,x,\xi)}$, whose local frequency in $\xi$ is $\partial_{\xi}\Phi_{\theta}$. Quadrature regularity is datum-dependent. For a general $u_0\in L^2$, its Fourier transform need only belong to $L^2$. The datum determines the location and regularity of $\wh u_0$; increasing a carrier wavenumber shifts its spectrum and may require a larger active window while leaving the target phase unchanged.

A composite trapezoidal rule resolves the integral when the phase and $g$ have sufficient regularity on the chosen bounded window and the frequency grid is fine enough for both. The experiments use this rule by default and use a piecewise-linear-phase Filon-type rule when specified \citep{iserles2005oscillatory,deano2018oscillatory}, which interpolates both the phase and the smooth factor linearly on each frequency panel; its panel formula is stated in Appendix~\ref{app:expdetails}. On a fixed frequency window and for the tested phases, that option reduces the observed dependence on the absolute oscillation count.

The error still depends on phase derivatives and stationary points, the regularity of $A_\theta\wh u_0$, the window size, and the omitted frequency tail; if the datum spectrum leaves the fixed window, the window or node count must change. The asymptotic behavior of Filon-type and stationary-phase-aware rules, including the degradation at a stationary point of the phase, is treated systematically by \citet{deano2018oscillatory}, and a rule matched to that analysis is the natural replacement for the default trapezoidal reconstruction. The trained advection and wave runs use the composite trapezoidal rule; the point-focus diagnostic of Sec.~\ref{sec:exp-caustic} uses the optional Filon rule. Node counts and windows are run-specific (Appendix~\ref{app:expdetails}).

\subsection{Symbol-dependent residual construction and operator-like reuse}
\label{sec:mino-symbol-residual}
\paragraph{Symbol-specific residual construction.} The objective in \eqref{eq:objective} depends on the symbol $p$ only through the differential operators that appear in \eqrefs{eq:Reik}{eq:Rtr}: the spatial gradient $\nabla_{x}\Phi_{\theta}$ inside $p(x,\nabla_{x}\Phi_{\theta})$, the group velocity evaluated at the current covector, and, for the half-density convention, its $x$-divergence and the specified subprincipal symbol. Given a symbolic expression for $p(x,\xi)$, the target-PDE normalization, and the quantization convention, symbolic differentiation yields the corresponding eikonal and transport residuals. Substitution of the exact homogeneous free-symbol generator supplies a constant-coefficient consistency check. For variable-coefficient transport with $p(x,\xi)=c(x)\xi$ in one dimension, this construction forms $\nabla_{x}\Phi_{\theta}=\partial_{x}\Phi_{\theta}$, substitutes into $p$ to get $c(x)\partial_{x}\Phi_{\theta}$, computes the group velocity $\nabla_{\xi}p=c(x)$, and uses the scalar normalization of the target advection equation. The resulting residuals are
\begin{align*}
\mathcal{R}_{\mathrm{eik}} = \partial_{t}\Phi_{\theta}+c(x)\partial_{x}\Phi_{\theta}, \qquad \mathcal{R}_{\mathrm{tr}} = \partial_{t}A_{\theta}+c(x)\partial_{x}A_{\theta}.
\end{align*}
The exact variable-coefficient characteristic phase $\Phi(t,x,\xi)=X(0;t,x)\xi$, where $\partial_sX(s;t,x)=c(X(s;t,x))$ and $X(t;t,x)=x$, with $A\equiv1$, annihilates these scalar residuals. If half-density transport is selected instead under the zero-subprincipal Weyl convention, the amplitude residual includes $c'(x)A_\theta/2$, thereby targeting the different PDE stated after \eqref{eq:Rtr}; a nonzero supplied subprincipal symbol adds its corresponding term. For the pure transport symbol $p(x,\zeta)=c(x)\zeta$, no separate $\partial_x^2\Phi_\theta A_\theta/2$ is appended. The geometric term is obtained once from the full divergence in \eqref{eq:Rtr}. That divergence is symbol-specific. For example, the mixed symbol $p(x,\zeta)=c(x)\zeta+\zeta^2/2$ gives $\operatorname{div}_x(v_\theta)=c'(x)+\partial_x^2\Phi_\theta$, so both contributions occur inside the single divergence term.

\paragraph{Operator-like reuse across initial conditions.} Different initial conditions enter only through $\wh u_0$ in \eqref{eq:fio}, so a generator trained for a given symbol applies to new data without retraining. MiNO therefore provides operator-like reuse across initial conditions under a fixed symbol. This reuse follows from the linearity of \eqref{eq:fio} in $\wh u_0$ and differs from amortized operator learning from input--output pairs. The distinction is tested directly in Sec.~\ref{sec:exp-multiic}.

\section{Theory of propagator learning}
\label{sec:theory}

The bicharacteristic relation for the localized homogeneous FIO target is Proposition~\ref{prop:causal}. Phase and reconstruction stability on a compact chart is developed in Sec.~\ref{sec:theory-phase}, residual control of the learned canonical graph in Sec.~\ref{sec:theory-canonical}, and scalar-transport reconstruction and frequency-tail separation in Secs.~\ref{sec:theory-transport} and~\ref{sec:theory-tail}. A fold normal-form illustration in Sec.~\ref{sec:theory-caustic} distinguishes a uniform oscillatory representation from its branchwise reduction. The proofs are collected in Appendix~\ref{app:proofdetails}.

One of these results certifies an object that previous error analyses do not address. Theorem~\ref{thm:canonical-residual} bounds the distance between the learned canonical relation and the exact one by differentiated eikonal residuals, so what is certified is the propagation geometry rather than a function. Residual-to-error and post-training certification estimates for neural PDE approximations have been developed for other equation classes and solution-field targets \citep{shin2020convergence,deryck2022error,mishra2023estimates, liu2023residual,eiras2024certification}, and approximation-theoretic bounds are available on the operator-learning side in terms of the regularity of the target solution operator and of the encoding and reconstruction maps that surround that operator \citep{kovachki2021universal,lanthaler2022deeponet}. All of these bound the error of a learned function; a geometric defect has no counterpart there.

The remaining results connect residuals to the reconstruction. The first controls phase error from the residual and then controls reconstruction from that phase estimate together with the amplitude error for a general real Hamilton--Jacobi chart. The scalar-transport results then control the complete phase--amplitude reconstruction and give a continuous-residual full-solution bound that includes the frequency tail, with numerical quadrature entering as an additional error component.

\subsection{Hamilton--Jacobi residual-to-phase and reconstruction stability}
\label{sec:theory-phase}
\label{sec:theory-residual}

The scalar value of the eikonal residual and its derivatives play different roles. The value controls the phase that appears in the oscillatory exponential, whereas its derivatives control the canonical graph. We first fix the common setting for a phase estimate and its two consequences.

Let $G\subset\R^{2d}$ be open, let $p\in C^2(G;\R)$, and fix $t\in[0,T]$, a bounded measurable output set $R\Subset\R^d$, and a compact frequency window $K\Subset\R^d$. Let $\phi$ and $\widetilde\phi$ be real $C^2$ phases on a common chart containing the tubes below, with
\begin{align*}
\partial_s\phi+p(z,\nabla_z\phi)=0, \qquad r\coloneqq\partial_s\widetilde\phi+p(z,\nabla_z\widetilde\phi).
\end{align*}
Put $e\coloneqq\widetilde\phi-\phi$ and assume that every line segment between $\nabla_z\phi(s,z,\xi)$ and $\nabla_z\widetilde\phi(s,z,\xi)$ remains in the fiber of $G$. Define the averaged velocity
\begin{equation}
\label{eq:averaged-velocity}
b(s,z,\xi)\coloneqq\int_0^1\nabla_\zeta p( z,\nabla_z\phi(s,z,\xi)+\lambda\nabla_z e(s,z,\xi) )\dd\lambda .
\end{equation}
For every $(x,\xi)\in R\times K$, assume that the terminal-value flow
\begin{equation}
\label{eq:averaged-flow}
\frac{\dd}{\dd s}Y_b(s;t,x,\xi) =b(s,Y_b(s;t,x,\xi),\xi), \qquad Y_b(t;t,x,\xi)=x,
\end{equation}
exists on $[0,t]$ and is a $C^1$ diffeomorphism in $x$. Set
\begin{align*}
R_s^b\coloneqq \{(Y_b(s;t,x,\xi),\xi):x\in R,\xi\in K\}, \qquad \kappa_b\coloneqq\norm{\operatorname{div}_z b}_{L^\infty(\cup_{0\le s\le t} \{s\}\times R_s^b)}.
\end{align*}

\begin{theorem}[Hamilton--Jacobi residual-to-phase stability]
\label{thm:hj-phase-stability}
In the preceding setting,
\begin{equation}
\label{eq:hj-phase-bound}
\norm{e(t)}_{L^2(R\times K)} \le \mathfrak P_t(e(0),r),
\end{equation}
where
\begin{equation}
\label{eq:phase-residual-functional}
\mathfrak P_t(e(0),r)\coloneqq e^{\kappa_b t/2}\norm{e(0)}_{L^2(R_0^b)} +\int_0^t e^{\kappa_b(t-s)/2} \norm{r(s)}_{L^2(R_s^b)}\dd s.
\end{equation}
\end{theorem}

\begin{corollary}[Reconstruction stability on a compact chart]
\label{cor:hj-reconstruction}
In the setting preceding Theorem~\ref{thm:hj-phase-stability}, let $A,\widetilde A$ be complex amplitudes at time $t$, with $A\in L^\infty(R\times K)$ and $\widetilde A-A\in L^2(R\times K)$. If $F^K_{\phi,A}$ and $F^K_{\widetilde\phi,\widetilde A}$ denote the reconstructions in \eqref{eq:fio} truncated to $K$, then
\begin{equation}
\label{eq:hj-reconstruction-bound}
\norm{F^K_{\widetilde\phi,\widetilde A}(t) -F^K_{\phi,A}(t)}_{L^2(\R^d)\to L^2(R)} \le (2\pi)^{-d/2}\left[ \norm{\widetilde A(t)-A(t)}_{L^2(R\times K)} +\norm{A(t)}_{L^\infty(R\times K)} \mathfrak P_t(e(0),r)\right].
\end{equation}
\end{corollary}

Thus the phase contribution in \eqref{eq:hj-reconstruction-bound} is controlled by the Hamilton--Jacobi residual, while the amplitude discrepancy is retained as the separate time-slice term $\norm{\widetilde A(t)-A(t)}_{L^2(R\times K)}$.

\begin{corollary}[Phase bound for exact Cauchy data]
\label{cor:hj-exact-data}
In the setting preceding Theorem~\ref{thm:hj-phase-stability}, suppose the phase Cauchy datum is exact. Then
\begin{equation}
\label{eq:hj-exact-data}
\mathfrak P_t(0,r) \le \Gamma_{\kappa_b}(t) \norm{r}_{L^2(\mathcal Q_b(t))}, \qquad \mathcal Q_b(t)\coloneqq\{(s,z,\xi):(z,\xi)\in R_s^b,0<s<t\},
\end{equation}
where
\begin{equation}
\label{eq:gamma-definition}
\Gamma_\kappa(t)\coloneqq
\begin{cases}
((e^{\kappa t}-1)/\kappa)^{1/2},&\kappa>0,\\
t^{1/2},&\kappa=0.
\end{cases}
\end{equation}
\end{corollary}

Theorem~\ref{thm:hj-phase-stability} and Corollary~\ref{cor:hj-reconstruction} are compact-chart stability results valid independently of positive homogeneity and mixed-Hessian nondegeneracy; Proposition~\ref{prop:causal} supplies the separate wavefront statement. The results also distinguish the roles of a scalar residual and its derivatives. The scalar residual controls the phase value in \eqref{eq:hj-reconstruction-bound}, while the differentiated residual controls the generated canonical graph in the next result.

\subsection{Eikonal-residual control of the canonical geometry}
\label{sec:theory-canonical}

We next separate the two trajectorywise mechanisms behind canonical-graph control. Let $G\subset\R^{2d}$ be open, let $p\in C^2(G;\R)$, and write
\begin{align*}
H_p(x,\zeta)\coloneqq(\nabla_\zeta p(x,\zeta), -\nabla_xp(x,\zeta)).
\end{align*}
Assume that $H_p$ is $L$-Lipschitz on a region containing all trajectories used below. Let $\widetilde\phi\in C^2([0,T]\times U\times V;\R)$ satisfy
\begin{align*}
\widetilde\phi(0,x,\xi)=x\cdot\xi, \qquad r\coloneqq\partial_t\widetilde\phi +p(x,\nabla_x\widetilde\phi).
\end{align*}
Fix $(t,x,\xi)$ in the phase chart and suppose that the terminal-value problem
\begin{equation}
\label{eq:approx-char}
\dot{\widetilde x}(s)= \nabla_\zeta p(\widetilde x(s), \nabla_x\widetilde\phi(s,\widetilde x(s),\xi)), \qquad \widetilde x(t)=x,
\end{equation}
exists for $0\le s\le t$ and remains in the phase chart, and that the exact comparison trajectories remain in the stated Lipschitz region. Set
\begin{align*}
\widetilde\zeta(s)\coloneqq \nabla_x\widetilde\phi(s,\widetilde x(s),\xi), \qquad y_0\coloneqq\widetilde x(0), \qquad \widetilde y\coloneqq\nabla_\xi\widetilde\phi(t,x,\xi).
\end{align*}

\begin{lemma}[Forced characteristic and footpoint estimates]
\label{lem:canonical-characteristic}
In the preceding setting,
\begin{align}
&\left| (x,\nabla_x\widetilde\phi(t,x,\xi)) -\chi_t(y_0,\xi) \right| \le \int_0^t e^{L(t-s)} \left|\nabla_xr(s,\widetilde x(s),\xi)\right|\dd s,
\label{eq:forced-flow-bound}\\
&\abs{\widetilde y-y_0} \le\int_0^t \left|\nabla_\xi r(s,\widetilde x(s),\xi)\right|\dd s.
\label{eq:footpoint-bound}
\end{align}
\end{lemma}

\begin{theorem}[Eikonal-residual control of the canonical graph]
\label{thm:canonical-residual}
In the setting preceding Lemma~\ref{lem:canonical-characteristic}, suppose additionally that $\partial^2_{x\xi}\widetilde\phi$ is nonsingular in the chart. Then the output point generated by the approximate phase satisfies
\begin{equation}
\label{eq:canonical-defect}
\left| (x,\nabla_x\widetilde\phi(t,x,\xi)) -\chi_t(\widetilde y,\xi) \right| \le \int_0^t e^{L(t-s)} \left|\nabla_xr(s,\widetilde x(s),\xi)\right|\dd s +e^{Lt}\int_0^t \left|\nabla_\xi r(s,\widetilde x(s),\xi)\right|\dd s.
\end{equation}
\end{theorem}

\begin{corollary}[Uniform canonical-graph residual bound]
\label{cor:canonical-uniform}
Under the hypotheses of Theorem~\ref{thm:canonical-residual}, let $\mathcal D_t$ be the set of $(x,\xi)$ for which its trajectory hypotheses hold. If the indicated residual derivatives are uniformly bounded, every point of the generated relation
\begin{align*}
\widetilde{\mathcal C}_t\coloneqq \left\{ (x,\nabla_x\widetilde\phi(t,x,\xi); \nabla_\xi\widetilde\phi(t,x,\xi),\xi) :(x,\xi)\in\mathcal D_t \right\}
\end{align*}
lies within output-phase-space distance
\begin{equation}
\label{eq:canonical-uniform}
\delta_t\coloneqq \frac{e^{Lt}-1}{L}\norm{\nabla_xr}_{L^\infty} +t e^{Lt}\norm{\nabla_\xi r}_{L^\infty}
\end{equation}
of $\operatorname{graph}\chi_t$, where the first coefficient is interpreted as $t$ when $L=0$, and both $L^\infty$ norms are taken over the union of the approximate trajectories associated with $(x,\xi)\in\mathcal D_t$.
\end{corollary}

Theorem~\ref{thm:canonical-residual} gives the trajectorywise estimate, and Corollary~\ref{cor:canonical-uniform} turns that estimate into uniform control of the geometric object carried by the phase.

\begin{remark}[Which residual norm certifies which object]
\label{rem:residual-norm}
The two statements identify the residual norm the geometry requires, and that norm differs from the norm the objective in \eqref{eq:objective} reports. A small scalar eikonal residual, without control of its $x$- and $\xi$-derivatives, does not by itself place the canonical graph near $\operatorname{graph}\chi_t$. Adding to the phase a function of time alone changes the scalar residual while leaving the graph unchanged, and the differentiated residual in \eqref{eq:canonical-defect} reflects that invariance correctly, whereas the scalar one does not. A reported training loss therefore certifies the reconstructed field, through Corollary~\ref{cor:hj-reconstruction}, but certifies the propagation geometry only through the differentiated norm. The higher-order Sobolev approximation bounds cited in Sec.~\ref{sec:mino-train} supply the approximation route to that norm.
\end{remark}

\subsection{Scalar-transport residual-to-propagator stability}
\label{sec:theory-transport}

Let $c\in C^1(\R^d;\R^d)$ generate a $C^1$ flow $X(s;t,x)$ on the tube considered below. Fix $t\in[0,T]$, a bounded measurable output set $R\Subset\R^d$, and a compact frequency window $K\Subset\R^d$. Put
\begin{align*}
R_s\coloneqq X(s;t,R),\quad 0\le s\le t, \qquad \kappa\coloneqq\sup_{0\le s\le t} \norm{\operatorname{div}c}_{L^\infty(R_s)}.
\end{align*}

\begin{lemma}[Characteristic-tube transport estimate]
\label{lem:tube-error}
Let $e$ be a real- or complex-valued $C^1$ function on the characteristic tube, set
\begin{align*}
r\coloneqq(\partial_s+c(z)\cdot\nabla_z)e, \qquad e_0(z,\xi)\coloneqq e(0,z,\xi),
\end{align*}
and define
\begin{equation}
\label{eq:residual-functional}
\mathfrak E_t(e_0,r)\coloneqq e^{\kappa t/2}\norm{e_0}_{L^2(R_0\times K)} +\int_0^t e^{\kappa(t-s)/2} \norm{r(s)}_{L^2(R_s\times K)}\dd s.
\end{equation}
Then
\begin{equation}
\label{eq:generator-error-from-residual}
\norm{e(t)}_{L^2(R\times K)}\le\mathfrak E_t(e_0,r).
\end{equation}
\end{lemma}

For the operator-level consequence, remain in the preceding characteristic-tube setting and put
\begin{align*}
\phi_c(s,z,\xi)\coloneqq X(0;s,z)\cdot\xi,
\end{align*}
so that $\phi_c$ is the exact scalar-transport phase and the exact amplitude is one. Let $\widetilde\phi$ be real-valued and let $\widetilde A$ be complex-valued, both $C^1$ on the characteristic tube, and define
\begin{align*}
r_\phi&\coloneqq\partial_s\widetilde\phi +c(z)\cdot\nabla_z\widetilde\phi, &e_{\phi,0}(z,\xi)&\coloneqq\widetilde\phi(0,z,\xi)-z\cdot\xi,\\
r_A&\coloneqq\partial_s\widetilde A +c(z)\cdot\nabla_z\widetilde A, &e_{A,0}(z,\xi)&\coloneqq\widetilde A(0,z,\xi)-1.
\end{align*}
Define the learned and exact truncated propagators by
\begin{align*}
\widetilde F^K(t)u_0(x) &\coloneqq(2\pi)^{-d/2}\int_K e^{i\widetilde\phi(t,x,\xi)}\widetilde A(t,x,\xi) \wh u_0(\xi)\dd\xi,\\
S_c^K(t)u_0(x) &\coloneqq(2\pi)^{-d/2}\int_K e^{i\phi_c(t,x,\xi)}\wh u_0(\xi)\dd\xi.
\end{align*}

\begin{theorem}[Residual-to-propagator stability for scalar transport]
\label{thm:transport-residual}
With the preceding definitions,
\begin{equation}
\label{eq:residual-operator-bound}
\norm{\widetilde F^K(t)-S_c^K(t)}_{L^2(\R^d)\to L^2(R)} \le(2\pi)^{-d/2} \left[\mathfrak E_t(e_{A,0},r_A) +\mathfrak E_t(e_{\phi,0},r_\phi)\right].
\end{equation}
\end{theorem}

\begin{corollary}[Propagator bound for exact Cauchy data]
\label{cor:transport-residual-l2}
Under the hypotheses of Theorem~\ref{thm:transport-residual}, suppose the phase and amplitude Cauchy data are imposed exactly, and let
\begin{align*}
\mathcal Q_t\coloneqq \{(s,z,\xi):0<s<t,z\in R_s,\xi\in K\}.
\end{align*}
Then
\begin{equation}
\label{eq:residual-operator-l2}
\norm{\widetilde F^K(t)-S_c^K(t)}_{L^2(\R^d)\to L^2(R)} \le(2\pi)^{-d/2}\Gamma_\kappa(t) \left(\norm{r_A}_{L^2(\mathcal Q_t)} +\norm{r_\phi}_{L^2(\mathcal Q_t)}\right),
\end{equation}
where $\Gamma_\kappa$ is defined in \eqref{eq:gamma-definition}.
\end{corollary}

\subsection{Full-solution error and frequency-truncation tail}
\label{sec:theory-tail}

The continuous residual estimate above is datum-uniform. Finite-collocation generalization and frequency truncation enter separately; the following corollary isolates the latter contribution.

\begin{corollary}[Full-solution bound with the frequency tail]
\label{cor:continuous-full}
Under the hypotheses of Corollary~\ref{cor:transport-residual-l2}, let $P_K$ be the Fourier projection $\widehat{P_Ku}\coloneqq\mathbf1_K\widehat u$, and let $S_c(t)u\coloneqq u\circ X(0;t,\cdot)$ be the full scalar-transport propagator. Define
\begin{equation}
\label{eq:continuous-residual-delta}
\Delta_{\mathrm{cont}}(t) \coloneqq(2\pi)^{-d/2}\Gamma_\kappa(t) \left(\norm{r_A}_{L^2(\mathcal Q_t)} +\norm{r_\phi}_{L^2(\mathcal Q_t)}\right).
\end{equation}
Then every $u_0\in L^2(\R^d)$ satisfies
\begin{equation}
\label{eq:full-solution-bound}
\norm{\widetilde F^K(t)u_0-S_c(t)u_0}_{L^2(R)} \le \Delta_{\mathrm{cont}}(t)\norm{u_0}_{L^2(\R^d)} +e^{\kappa t/2}\norm{(I-P_K)u_0}_{L^2(\R^d)}.
\end{equation}
\end{corollary}

Corollary~\ref{cor:continuous-full} concerns the continuous oscillatory integral $\widetilde F^K$. For a numerical quadrature approximation $\widetilde F^{K,Q}$, the full error additionally contains $\norm{(\widetilde F^{K,Q}(t)-\widetilde F^K(t))u_0}_{L^2(R)}$.

\begin{proposition}[Sharp square-wave truncation tail]
\label{prop:square-tail}
For $d=1$, fixed $a>0$, $\Xi>0$, $K=[-\Xi,\Xi]$, and the square-wave datum $u_0=\mathbf1_{[-a,a]}$, the Fourier-tail norm in Corollary~\ref{cor:continuous-full} has the explicit bound
\begin{equation}
\label{eq:square-tail-bound}
\norm{(I-P_K)u_0}_{L^2(\R)}\le\frac{2}{\sqrt{\pi\Xi}},
\end{equation}
and, writing $\operatorname{Si}(s)\coloneqq\int_0^s(\sin q)/q\dd q$, the exact relative identity
\begin{equation}
\label{eq:square-tail-exact}
\frac{\norm{(I-P_K)u_0}_{L^2(\R)}}{\norm{u_0}_{L^2(\R)}} =\left\{\frac{2}{\pi}\left[ \frac{\sin^2(a\Xi)}{a\Xi}+\frac{\pi}{2} -\operatorname{Si}(2a\Xi) \right]\right\}^{1/2}.
\end{equation}
Consequently, the sharp relative asymptotics are
\begin{equation}
\label{eq:square-tail-asymptotic}
\frac{\norm{(I-P_K)u_0}_{L^2(\R)}}{\norm{u_0}_{L^2(\R)}} =\frac{1}{\sqrt{\pi a\Xi}}+O(\Xi^{-3/2}) \qquad(\Xi\to\infty).
\end{equation}
\end{proposition}

The bound in \eqref{eq:residual-operator-l2} is uniform over all $u_0\in L^2(\R^d)$ and therefore does not require the datum to be smooth; in particular, that bound covers the discontinuous-advection datum after frequency truncation. The $\Xi^{-1/2}$ rate in \eqref{eq:square-tail-asymptotic} is the $L^{2}$ form of the classical obstruction to spectral approximation of a discontinuous datum, whose pointwise counterpart is the Gibbs phenomenon \citep{gottlieb1997gibbs}; that rate is a property of the truncated reconstruction window in \eqref{eq:fio}, not of the learned generator, which is why Corollary~\ref{cor:continuous-full} keeps the two terms separate. Section~\ref{sec:exp-discont} evaluates this tail at the parameters of the discontinuous-advection benchmark and compares that tail with the observed error floor.

\subsection{A fold normal-form illustration}
\label{sec:theory-caustic}

We use the Airy normal form to describe a fold of the base projection of a propagated solution Lagrangian. The propagator canonical relation remains the operator-level object of Proposition~\ref{prop:causal}. The systematic theory behind such normal forms, covering oscillatory integrals attached to Lagrange immersions and the unfolding of singularities of their base projections, is classical \citep{duistermaat1974oscillatory}.

\begin{example}[The Airy fold, explicitly]
\label{ex:airy}
The canonical fold is generated by $\phi(x,\xi)=x\xi-\xi^{3}/3$. Then $\partial_{\xi}\phi=x-\xi^{2}$, so the stationary set is $x=\xi^{2}$. Its projection $\xi\mapsto x=\xi^2$ has a fold at $\xi=0$, whose caustic value is $x=0$, and $\partial^{2}_{\xi\xi}\phi=-2\xi$ has a simple zero at $\xi=0$. The generator $\phi$ is a polynomial, hence $C^{\infty}$ everywhere, including at the fold. In this specified normal-form chart, the Hessian represents the projection degeneracy on the critical set. The branchwise stationary-phase coefficient contains the factor $\abs{\partial^{2}_{\xi\xi}\phi}^{-1/2}=\abs{2\xi}^{-1/2}$ on the folding branches. Since $x=\xi^2$, its illuminated-side base-distance scaling is $2^{-1/2}x^{-1/4}$ as $x\to0^{+}$. The Airy integral gives the uniform description at the fold, while these branchwise asymptotics apply away from the fold; higher unfoldings belong to the oscillatory-integral theory of \citet{duistermaat1974oscillatory}.
\end{example}

\begin{remark}[Uniformized and branchwise fold behavior]
\label{rem:fold-contrast}
Example~\ref{ex:airy} separates the smooth polynomial phase from the branchwise coefficient produced by nondegenerate stationary phase away from the fold. For a properly supported order-zero FIO associated with a canonical graph, the standard FIO mapping theorem gives an $L^2\to L^2$ operator bound \citep[Ch.~25]{hormander1985vol4}. Ordinary stationary phase supplies the branchwise pointwise asymptotics away from the degenerate critical point.
\end{remark}

Example~\ref{ex:airy} and Remark~\ref{rem:fold-contrast} clarify the local geometry. The point-focus diagnostic of Sec.~\ref{sec:exp-caustic} uses a configuration in which all rays collapse, complementing this fold normal form. To fix the orders of magnitude, the standard normalized one-dimensional fold family, with $A\in C_c^\infty(\R)$ and $A(0)\ne0$,
\begin{align*}
(2\pi\varepsilon)^{-1/2}\int e^{i(x\xi-\xi^3/3)/\varepsilon}A(\xi)\dd\xi
\end{align*}
has an Airy transition of width $\varepsilon^{2/3}$, across which the integral reaches size $\varepsilon^{-1/6}$ and its first $x$-derivative size $\varepsilon^{-5/6}$. The oscillatory integral is smooth for every fixed $\varepsilon>0$; its Maslov change is encoded in the WKB branch representation.

\section{Experiments}
\label{sec:experiments}

The twelve studies are organized into three groups: propagation regimes; representational content, reuse, and extension; and efficiency and supervision. The studies comprise trained-generator benchmarks and exact-generator diagnostics. The latter isolate reconstruction through a point focus and resolved wavefront transport; the remaining studies evaluate learned generators, reuse, inverse recovery, dimensional extension, efficiency, and supervision. The two entries of Sec.~\ref{sec:exp-efficiency-supervision} analyze runs already reported in Sec.~\ref{sec:exp-regimes} rather than adding training runs. Appendix~\ref{app:expdesign} pairs each study with its objective and benchmark and states the common protocol, baselines, and computing environment. Errors are relative $L^{2}$ errors, and trained MiNO results are reported as the mean and sample standard deviation over five seeds unless stated otherwise.

\subsection{Propagation regimes}
\label{sec:exp-regimes}

\subsubsection{Smooth advection: closed-form calibration}
\label{sec:exp-smooth}

A problem with no caustic and a known closed form calibrates the combined learned phase--amplitude model and numerical reconstruction against the exact phase-increment target $-\xi$.

\paragraph{Protocol.} Linear advection $\partial_{t}u+\partial_{x}u=0$ is evaluated on $x\in[-3,3]$, $t\in[0,1.5]$, with the decaying Gaussian initial condition $u_{0}(x)=e^{-x^{2}/(2\sigma^{2})}$, $\sigma=0.3$, and the whole-line exact solution $u(t,x)=u_{0}(x-t)$. The box is chosen so the translated Gaussian is negligible at its boundary over the evaluation interval; no periodic identification is used for the MiNO reference. The symbol is $p(x,\xi)=\xi$, spatially homogeneous, so the exact phase is $x\xi-t\xi$ and the target for the learned phase increment $h_{\theta}$ is $-\xi$. Reconstruction uses the default composite trapezoidal rule (Appendix~\ref{app:expdetails}). MiNO is trained for 50,000 steps over five seeds. The supervised FNO is trained over the same five seeds on its own labeled dataset and under its own epoch budget (Appendix~\ref{app:baseline-config}); step counts are not comparable across the two training modes. The two field baselines, the NTK-balanced PINN and the causal PINN of \citet{wang2022causal}, enter at the 200,000-step cross-paradigm budget, at which MiNO is retrained over the same five seeds. The metric is the relative $L^{2}$ error at the final time $t=1.5$.

\paragraph{Results.} At 50,000 steps, MiNO reaches $3.84\times10^{-3}\pm2.18\times10^{-3}$, and the supervised FNO, at its own budget, reaches $3.12\times10^{-2}\pm4.27\times10^{-3}$. At the 200,000-step cross-paradigm budget (Figure~\ref{fig:crosspar}), MiNO reaches $1.32\times10^{-3}\pm4.36\times10^{-4}$, the NTK-balanced PINN reaches $5.48\times10^{-3}\pm2.04\times10^{-3}$, and the causal PINN reaches $6.09\times10^{-1}\pm2.07\times10^{-1}$.

\paragraph{Interpretation.} Without labeled solution trajectories, MiNO attains an error about eight times smaller than that of the supervised FNO on this benchmark; that contrast, and the data-distribution caveat attached to it, are examined directly in Sec.~\ref{sec:exp-fno}. In this configuration, the causal PINN exhibits a substantially larger mean error and greater seed-to-seed variability than the NTK-balanced PINN.

\begin{figure}[pos=t!]
\centering
\includegraphics[width=\linewidth]{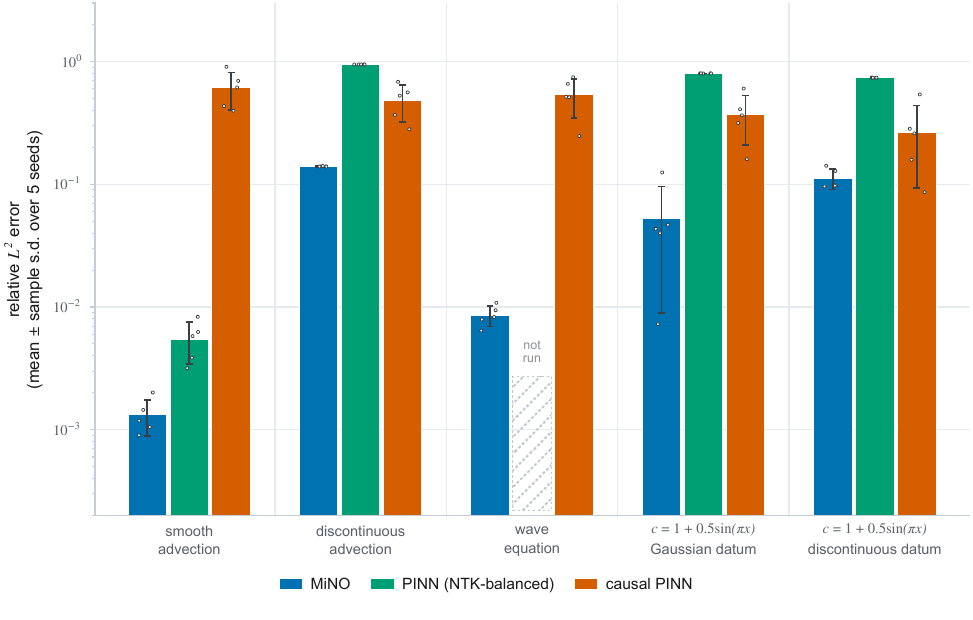}
\caption{Cross-paradigm comparison on the 200,000-step runs, plotted on a log scale, for MiNO, the NTK-balanced PINN, and the causal PINN across smooth advection, discontinuous advection, wave propagation, and the Gaussian and discontinuous initial-data cases for $c(x)=1+0.5\sin(\pi x)$ of Sec.~\ref{sec:exp-varcoeff-sin}. Those two variable-coefficient groups are distinct from the $c(x)=-x$ stationary-point study of Sec.~\ref{sec:exp-varcoeff}, which is not plotted here. All plotted errors use their respective final-time protocols. For the wave equation, the comparison includes MiNO and the causal PINN, the two implementations that enforce the zero-initial-velocity condition. The supervised FNO is evaluated separately.}
\label{fig:crosspar}
\end{figure}

\subsubsection{Discontinuous advection: a matched-budget comparison}
\label{sec:exp-discont}

On sharp transport, field learning and propagator learning occupy different regimes over the budget sweep. MiNO reaches the frequency-truncation floor of its reconstruction window while the NTK-balanced PINN remains near its initial error level.

\paragraph{Protocol.} Linear advection $\partial_{t}u+\partial_{x}u=0$ is evaluated on $x\in[-3,3]$, $t\in[0,1.5]$, with $u_0(x)=\mathbf 1_{[-0.3,0.3]}(x)$ and exact whole-line solution $u(t,x)=u_0(x-t)$. The symbol is $p(x,\xi)=\xi$; the reconstruction uses 1024 frequency nodes on $[-100,100]$ because the square-wave spectrum decays only like $\abs{\xi}^{-1}$. MiNO and the NTK-balanced PINN are trained over five seeds at budgets of 500, 1,000, 2,000, 5,000, 10,000, 20,000, and 50,000 steps. The metric is the discrete relative $L^{2}$ error over the full $21\times256$ space--time evaluation grid; we additionally report the log--log ordinary-least-squares slope and the ratio of mean errors.

\paragraph{Results.} Table~\ref{tab:learncurve} reports the full sweep with standard deviations. MiNO falls from $0.299\pm0.154$ at 500 steps to $0.105\pm0.001$ at 10,000 steps and $0.103\pm0.001$ at 50,000 steps. The change after 10,000 steps is $0.002$, against a reduction of $0.194$ over the preceding budgets, so the error floor is reached by 10,000 steps. The PINN stays at $0.943\pm0.011$ at 500 steps, $0.953\pm0.011$ at 10,000 steps, and $0.950\pm0.003$ at 50,000 steps. Its error is essentially flat over this range. The log--log convergence slope is $-0.21$ for MiNO and $+0.001$ for the PINN. The ratio of mean errors is $3.15\times$ at 500 steps and $9.19\times$ at 50,000 steps.

The 200,000-step final-time benchmark (Figure~\ref{fig:crosspar}) gives MiNO $0.140\pm0.001$, the NTK-balanced PINN $0.948\pm4\times10^{-5}$, and the causal PINN $0.484\pm0.160$, all at $t=1.5$. The 50,000-step sweep reports full-grid error, whereas this 200,000-step benchmark reports final-time error.

\begin{table}[pos=t!]
\centering
\caption{Discontinuous-advection learning curve for MiNO and the NTK-balanced PINN, evaluated over the full space--time grid.}
\label{tab:learncurve}
\small
\begin{tabular}{@{}rccc@{}}
\toprule
Budget (thousand steps) & MiNO & PINN & PINN-to-MiNO ratio ($\times$) \\
\midrule
$0.5$ & $0.299\pm0.154$ & $0.943\pm0.011$ & $3.15$ \\
$1$   & $0.179\pm0.051$ & $0.949\pm0.012$ & $5.31$ \\
$2$   & $0.136\pm0.038$ & $0.953\pm0.014$ & $7.03$ \\
$5$   & $0.127\pm0.049$ & $0.954\pm0.014$ & $7.53$ \\
$10$  & $0.105\pm0.001$ & $0.953\pm0.011$ & $9.09$ \\
$20$  & $0.105\pm0.003$ & $0.952\pm0.007$ & $9.09$ \\
$50$  & $0.103\pm0.001$ & $0.950\pm0.003$ & $9.19$ \\
\bottomrule
\end{tabular}
\end{table}

\paragraph{Interpretation.} Within this sweep, the two methods occupy different convergence regimes. The fitted PINN slope is essentially zero, with no material change in error as the optimization budget increases. The square wave is a transported wavefront feature, a discontinuity whose wavefront set is propagated by the bicharacteristic flow. MiNO has a clear negative slope. For the generator, that same discontinuity is carried by the phase increment $h^{*}=-\xi$, linear in the fiber variable and constant in $t$ and $x$, together with unit amplitude. The feature the field target must resolve is absent from the object MiNO learns.

Proposition~\ref{prop:square-tail} shows that truncating the square-wave spectrum at $\Xi=100$ with datum half-width $a=0.3$ leaves an exact relative $L^2$ tail of $0.1028$; its leading asymptotic value is $(30\pi)^{-1/2}=0.1030$. The analytic tail is therefore numerically close to the late-budget mean $0.103$, supporting the interpretation that the selected frequency window sets the leading scale of the observed late-budget error. The discrete metric additionally contains phase, amplitude, and quadrature error.

The diagnostic of \citet{wang2022ntk} applies. Inspection of the PINN run shows the PDE residual becoming numerically negligible while the initial-condition loss stalls. The network has converged to the trivial solution $u\equiv0$, which satisfies $\partial_{t}u+\partial_{x}u=0$ everywhere. This behavior is consistent with documented difficulties of strong-form residual training on hyperbolic transport carrying a sharp front \citep{fuks2020limitations,krishnapriyan2021failure}. The present comparison adds a matched-budget learning curve against a propagator-learning target on the same problem.

That failure mode is unavailable to MiNO, and not because of tuning. The collapse requires a competition between an equation residual and a separate initial-condition penalty, and the time-factored parameterization in \eqref{eq:phase} removes the second term: $\Phi_\theta(0,\cdot,\cdot)=x\cdot\xi$ and $A_\theta(0,\cdot,\cdot)=1$ hold identically in $\theta$, so $\mathcal B_0$ in \eqref{eq:objective} vanishes for every parameter value (Sec.~\ref{sec:mino-param}). Driving the eikonal and transport residuals to zero cannot trade the Cauchy datum away, because that datum is not something the optimizer holds. What separates the two methods here is the parameterized target itself, not the optimizer applied to it.

\subsubsection{Wave propagation: bidirectional dynamics}
\label{sec:exp-wave}

MiNO handles the second-order wave equation through a two-branch generator, without leaving the propagator-learning paradigm; the canonical relation there has two branches, the left and right movers.

\paragraph{Protocol.} For the one-dimensional wave equation $u_{tt}=u_{xx}$ with zero initial velocity, the d'Alembert decomposition used here is
\begin{align*}
u(t,x)=\tfrac12(F_+(t)u_0)(x) +\tfrac12(F_-(t)u_0)(x) =\tfrac12[u_0(x-t)+u_0(x+t)].
\end{align*}
The normalized branches have phases $(x-t)\xi$ and $(x+t)\xi$ and branch amplitudes $A_\pm\equiv1$; the explicit external weights $1/2$ ensure that their sum equals $u_0$ at $t=0$. The evaluation domain is $x\in[-3,3]$, $t\in[0,1.5]$, with initial displacement $u_0(x)=e^{-x^2/(2\sigma^2)}$, $\sigma=0.2$, and zero initial velocity. The two characteristic families are carried by two single-branch generators, so the parameter count doubles relative to single-branch MiNO (Table~\ref{tab:params}). MiNO is trained over five seeds at both the 50,000-step and 200,000-step budgets. The 50,000-step MiNO metric is the discrete relative $L^{2}$ error on the full $21\times256$ space--time grid; the 200,000-step cross-paradigm metric is the final-time relative $L^{2}$ error at $t=1.5$. The causal PINN supplies the wave baseline because that baseline enforces both the initial displacement and the initial velocity.

\paragraph{Results.} At 50,000 steps, MiNO reaches $8.29\times10^{-3}\pm2.90\times10^{-3}$ on the full grid. Under the 200,000-step final-time protocol (Figure~\ref{fig:crosspar}), MiNO reaches $8.57\times10^{-3}\pm1.65\times10^{-3}$ and the causal PINN is at $0.536\pm0.189$.

\paragraph{Interpretation.} Each branch carries its own eikonal phase and transport amplitude, so the two-branch superposition is the simplest instance of the phase-branch management that higher-dimensional propagation requires. The mean error is about 63 times smaller than that of the causal baseline at the same final-time budget and is stable across seeds at the $10^{-3}$ level.

\subsubsection{Variable-coefficient transport with a stationary point}
\label{sec:exp-varcoeff}

This experiment tests the learned eikonal generator for one inhomogeneous scalar-transport symbol whose velocity vanishes at an interior point.

\paragraph{Protocol.} The problem is $\partial_tu-x\partial_xu=0$ as a whole-line Cauchy problem evaluated on $x\in[-2.5,2.5]$, $t\in[0,1.5]$, with $u_0(x)=(2\pi\sigma^2)^{-1/2}e^{-x^2/(2\sigma^2)}$, $\sigma=0.3$, and symbol $p(x,\xi)=-x\xi$. Its flow is $x(t;y)=e^{-t}y$, its projection Jacobian is $e^{-t}$, and hence
\begin{align*}
  u(t,x)=u_0(e^t x),\qquad \Phi(t,x,\xi)=e^t x\xi,\qquad A\equiv1 .
\end{align*}
The run uses five seeds (Appendix~\ref{app:expdetails}) and reports relative $L^2$ at $t\in\{0.5,1.0,1.5\}$. Its training budget is 20,000 steps.

\paragraph{Results.} The relative $L^{2}$ is $8.89\times10^{-3}\pm3.09\times10^{-3}$ at $t=0.5$, $1.64\times10^{-2}\pm5.88\times10^{-3}$ at $t=1.0$, and $2.37\times10^{-2}\pm6.81\times10^{-3}$ at $t=1.5$.

\paragraph{Interpretation.} The errors grow with time while remaining below $0.024$ in this window. Differentiating the symbol $p=c(x)\xi$ supplies the eikonal residual by the construction of Sec.~\ref{sec:mino-symbol-residual}, whose analytic check is the characteristic phase $\Phi=X(0;t,x)\xi$ with scalar amplitude $A\equiv1$; the eikonal closure therefore follows from the coefficient without being specified by hand, and the run uses the constant-amplitude specialization of Sec.~\ref{sec:mino-train}. The velocity vanishes at $x=0$, while $\partial_yx(t;y)=e^{-t}>0$ for every finite $t$, so the ray projection stays nonsingular and the test isolates a stationary point of the symbol from a projection singularity.

\subsubsection{Variable-coefficient transport with a sinusoidal speed}
\label{sec:exp-varcoeff-sin}

The second inhomogeneous symbol carries the cross-paradigm comparison of Figure~\ref{fig:crosspar} for both a smooth and a discontinuous initial datum, so the effect of initial-data regularity is separated from the effect of the variable coefficient.

\paragraph{Protocol.} The problem is $\partial_tu+c(x)\partial_xu=0$ with $c(x)=1+0.5\sin(\pi x)$ on $x\in[-1,1]$, one spatial period of the coefficient, and $t\in[0,1]$. The reference is the characteristic pullback $u(t,x)=u_0(X(0;t,x))$ evaluated by numerical backward integration of $\dot X=c(X)$. Two initial data are used: the Gaussian $u_0(x)=e^{-25(x+0.5)^2}$ and the square wave $u_0=\mathbf 1_{[-0.8,-0.2]}$. Reconstruction uses the default composite trapezoidal rule on a wider frequency window for the square wave than for the Gaussian datum, again reflecting the $\abs{\xi}^{-1}$ spectral decay (Appendix~\ref{app:expdetails}). MiNO, the NTK-balanced PINN, and the causal PINN are trained over five seeds at the 200,000-step cross-paradigm budget; the metric is final-time relative $L^{2}$ at $t=1$ on the 256-point spatial grid.

\paragraph{Results.} For the Gaussian datum, MiNO reaches $0.052\pm0.043$, the NTK-balanced PINN $0.801\pm0.001$, and the causal PINN $0.370\pm0.160$. For the square wave, MiNO reaches $0.112\pm0.021$, the NTK-balanced PINN $0.739\pm0.001$, and the causal PINN $0.265\pm0.173$. Figure~\ref{fig:crosspar} plots both groups alongside the constant-coefficient cases.

\paragraph{Interpretation.} A spatially varying speed makes the exact phase $X(0;t,x)\xi$ nonlinear in $x$, so the generator is no longer the polynomial phase of the constant-coefficient benchmarks. MiNO yields lower errors than both field baselines for both initial data, with error ratios ranging from approximately $2.4$ to $15.4$. The MiNO seed spread is the widest recorded in this study, a factor of about 17 between the best and worst seeds on the Gaussian datum. Accuracy here is set by the seed-to-seed variance of the learned phase, not by the reconstruction.

\subsubsection{A point focus: exact-generator reconstruction}
\label{sec:exp-caustic}

A caustic makes the base projection of a propagated solution Lagrangian singular while a uniformizing oscillatory representation can remain regular. This experiment evaluates that representation through a point focus.

\paragraph{Protocol.} We reconstruct through the focus using the exact phase $\phi(t,x,\xi)=x\xi-t\xi^2/2$ for the free Schr\"odinger multiplier $p(\xi)=\xi^2/2$ in \eqref{eq:fio} with $A\equiv1$, and compare its reconstruction with a high-resolution spectral solution. The datum is a two-bump amplitude carrying the converging chirp $\exp(-iy^2/(2f))$,
\begin{align*}
u_0(y)=\left[e^{-(y-\delta)^{2}/(2\sigma^{2})} +0.6e^{-(y+\delta)^{2}/(2\sigma^{2})}\right]e^{-iy^{2}/(2f)}, \qquad \sigma=0.3,\quad \delta=0.4\sigma,\quad f=1,
\end{align*}
whose unequal bumps keep the exact solution outside the closed-form chirped Gaussian family. Its rays satisfy
\begin{align*}
x(t;y)=(1-t/f)y, \qquad J_{\mathrm{proj}}(t,y)=1-t/f.
\end{align*}
Thus $J_{\mathrm{proj}}$ vanishes at $t=f$. Because the ray map is linear in $y$, all rays collapse simultaneously at $t=f$, producing a point-focus blowdown. The reconstruction uses the Filon rule of Appendix~\ref{app:expdetails} with 4096 frequency nodes on $[-30,30]$, outside which the two-bump spectrum is negligible, and $\wh u_0$ is computed by quadrature over $y\in[-6,6]$ with 8192 nodes. The reference is the periodic spectral solution on the box $[-L/2,L/2)$ at the fixed spacing $\Delta x=1/256$, for the three widths $L\in\{16,32,48\}$; both fields are compared on $\abs{x}\le3$ at $t/f\in\{0,0.5,0.9,1,1.1,1.5,2\}$. Since the free evolution spreads the datum, a narrow box wraps the datum periodically, and the reference, not the reconstruction, is what $L$ controls.

\paragraph{Results.} The reconstruction tracks the spectral solution through and past the focal time, with relative $L^2$ below $10^{-3}$ at every sampled time once the reference box is wide enough to remove periodic wrap-around (Figure~\ref{fig:robust}).

\paragraph{Interpretation.} At finite wavelength, the oscillatory integral remains finite and smooth at every sampled time, including the focal time. The singularity therefore lives in the base projection of the solution Lagrangian, not in the uniform reconstruction.

\begin{figure}[pos=t!]
\centering
\includegraphics[width=0.60\linewidth]{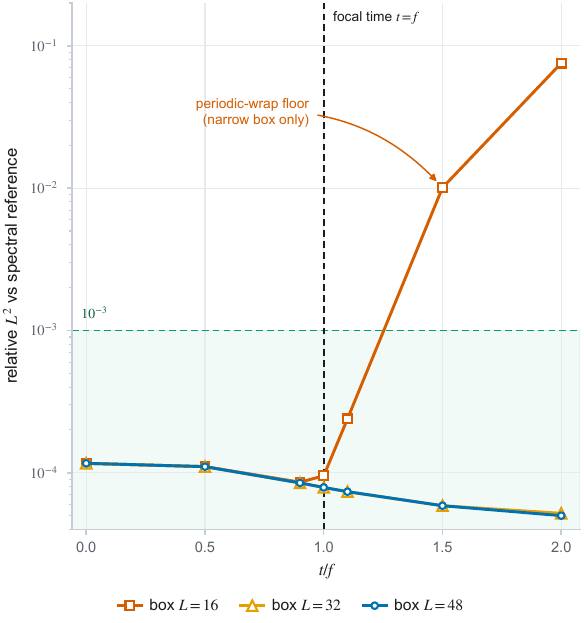}
\caption{Point-focus reconstruction from the exact generator, with no network trained, plotted against the periodic spectral reference computed on the box $[-L/2,L/2)$. Widening that box removes the periodic-wrap floor; at $L=32$ and $L=48$ the reconstruction stays below $10^{-3}$ relative $L^{2}$ through and past the focal time $t=f$.}
\label{fig:robust}
\end{figure}

\subsection{Representational content, reuse, and extension}
\label{sec:exp-representation}

\subsubsection{Finite-resolution microlocal fidelity}
\label{sec:exp-spectral}

A reconstruction can match coarse values while losing resolved high-frequency localization. The generator reconstruction preserves the transported ridge in a fixed-resolution wave-packet diagnostic.

\paragraph{Protocol.} A propagated discontinuity is examined with one fixed Fourier--Bros--Iagolnitzer-type wave-packet transform
\begin{align*}
T_{\lambda}v(x_{0},\xi_{0})\coloneqq(\lambda/\pi)^{1/4} \int e^{i\xi_{0}(x-x_{0})}e^{-\lambda(x-x_{0})^{2}/2}v(x)\dd x, \qquad \lambda=40 ,
\end{align*}
whose squared modulus is the portrait plotted in Figure~\ref{fig:portrait}. The datum is $u_0=\mathbf 1_{[-0.4,0]}$ under unit-speed advection, sampled at $t=0.5$ on 1024 points of $x\in[-1,1]$, so the two transported jumps sit at $x=0.1$ and $x=0.5$. The portrait grid carries 80 points on $x_{0}\in[-0.8,0.8]$ and 200 points on $\xi_{0}\in[-260,260]$. Three fields are compared on that grid: the exact grid values; the reconstruction from the exact generator with 8192 trapezoidal nodes on $[-300,300]$; and a low-pass surrogate that keeps the first 12 discrete Fourier modes of the same grid, \ie, $\abs{\xi}\lesssim35$.

\paragraph{Results.} The low-frequency content agrees across all three portraits. In the resolved high-frequency band associated with the transported discontinuity, the generator reconstruction follows the ridge of the exact solution while the low-pass surrogate does not (Figure~\ref{fig:portrait}).

\paragraph{Interpretation.} At this grid and wave-packet scale, the diagnostic separates value agreement from microlocal agreement. The generator reconstruction carries content that a value-matched low-pass surrogate does not, namely the resolved high-frequency band attached to the transported discontinuity.

\begin{figure}[pos=t!]
\centering
\includegraphics[width=\linewidth]{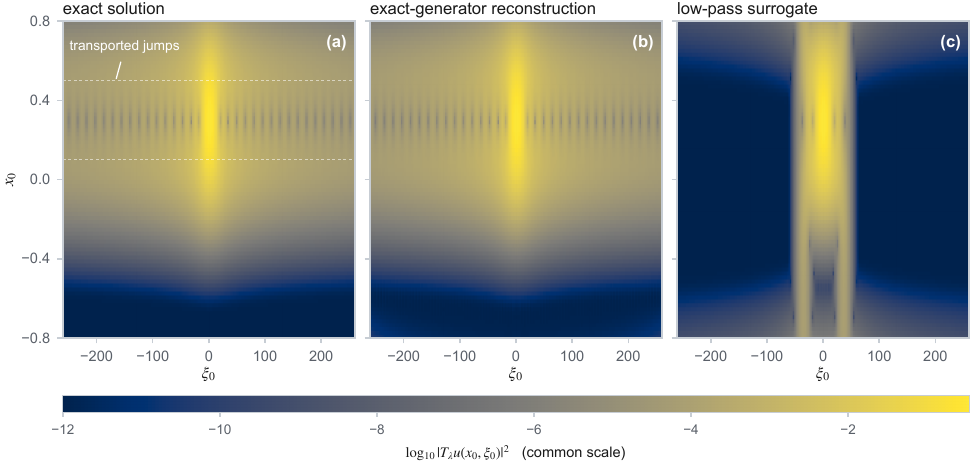}
\caption{Fixed-resolution wave-packet portrait $\log_{10}\abs{T_{\lambda}u(x_{0},\xi_{0})}^{2}$ of a propagated discontinuity at $\lambda=40$, on a common color scale: (a) the exact grid values, with the two transported jumps marked; (b) the reconstruction from the exact finite-window generator, with no network trained; and (c) a low-pass surrogate keeping the first 12 Fourier modes. The generator reconstruction preserves the resolved high-frequency ridge attached to those jumps, whereas the surrogate removes that ridge while retaining coarse content.}
\label{fig:portrait}
\end{figure}

\subsubsection{Multi-initial-condition propagator reuse}
\label{sec:exp-multiic}

A generator trained once for a fixed symbol applies to structurally distinct initial conditions without retraining, because the initial condition enters \eqref{eq:fio} only through $\wh u_0$.

\paragraph{Protocol.} A single MiNO generator is trained for 30,000 steps with one seed, using only residuals for the free-particle symbol $p(\xi)=\xi^{2}/2$. No initial datum enters that objective. The generator is then evaluated, with no further optimization, on five structurally distinct initial conditions: Gaussian, chirp-Gaussian, double-Gaussian, smooth step, and cosine-envelope. The metric is the relative $L^{2}$ error before, at, and after $t=0.5$.

\paragraph{Results.} At $t=0.5$, the relative $L^{2}$ is $0.080$ for the Gaussian, $0.096$ for the chirp-Gaussian, $0.070$ for the double-Gaussian, $0.077$ for the smooth step, and $0.105$ for the cosine-envelope, with mean $0.086$ across the five; the Gaussian serves as the calibration case.

\paragraph{Interpretation.} The generator is reused across initial conditions at comparable error levels, with the largest recorded error on the cosine-envelope. This reuse realizes the linearity of \eqref{eq:fio} in $\wh u_0$. The learned object is the symbol's generator, and the datum is supplied to the reconstruction quadrature, not to the network. For the chirp-Gaussian, the evaluation time also probes focal behavior, since that datum contains the two quadratic-phase components $e^{\pm ix^2/(2f)}$ and the minus component has ray map $x(t;y)=(1-t/f)y$, whose projection collapses at $t=f=0.5$; the other data have different phase geometry. Section~\ref{sec:exp-caustic} isolates the focus in the point-focus diagnostic.

\subsubsection{Inverse transport-speed recovery}
\label{sec:exp-inverse}

The learned propagator can serve as a structured forward model within a parameter-recovery problem, extending its use beyond forward solving.

\paragraph{Protocol.} The forward problem is $\partial_{t}u+c\partial_{x}u=0$ on $x\in[0,3]$, $t\in[0,1]$, with a Gaussian initial condition of width $0.15$ centered at $0.5$. Observations are 8 time points by 16 spatial points, 128 in total, with additive noise of standard deviation $\sigma\in\{0,0.01,0.02,0.05\}$ over five noise seeds, for true speeds $c\in\{0.9,1.3,1.5\}$. Three recovery methods are compared. M1 is a bank of 21 MiNO forward models pre-trained at candidate speeds $c\in\{0.80,0.85,\dots,1.80\}$ for 5,000 steps each, recovering $c$ by the argmin of the data mean squared error. M2 is an analytic forward-model oracle that performs gradient descent on $\log c$ using the exact closed-form solution; this method is an oracle because the exact closed-form solution must be known in advance. M3 is a PINN inverse that jointly optimizes network weights and $c$ over 10,000 steps. The metric is the mean absolute error in $c$ over the fifteen combinations of true speed and seed at each noise level.

\paragraph{Results.} The three true speeds lie on the $0.05$-spaced M1 candidate grid, so recovery by M1 is a grid-node identification rather than a continuous estimate. M1 selects the true-speed node in all sixty configurations, the sampled array of 21 misfit values having a unique minimum there in every run, and its mean absolute error in $c$ therefore vanishes at every tested noise level; that value is not comparable with the continuous estimates below. For the analytic oracle M2, the mean absolute error is $0.0000$, $0.0008\pm0.0004$, $0.0016\pm0.0008$, and $0.0039\pm0.0020$ at $\sigma=0,0.01,0.02,0.05$, and for the PINN inverse M3, the mean absolute error is $0.0004\pm0.0004$, $0.0018\pm0.0015$, $0.0036\pm0.0025$, and $0.0100\pm0.0064$ at the same noise levels.

\paragraph{Interpretation.} The pre-trained generator bank turns inverse recovery into a one-parameter grid-node identification problem, whereas M2 and M3 estimate a continuous parameter and their errors track the noise level. Off-grid resolution for M1 is determined by the candidate spacing and can be increased through grid refinement or local optimization. The experiment thus demonstrates structured one-parameter recovery with a pre-trained forward-model bank.

\subsubsection{Two-dimensional zero-velocity wave propagation}
\label{sec:exp-2d}

This experiment carries the same phase-space parameterization and reconstruction to $T^*(\R^2)$ for a zero-velocity wave solution.

\paragraph{Protocol.} The run uses a single half-wave branch, whose chosen principal symbol is $p_-(\xi)=-\abs{\xi}$ for $\xi=(\xi_1,\xi_2)$. The run therefore targets $\Phi=x\cdot\xi+t\abs{\xi}$ on $T^{*}(\R^{2})=\R^{4}$ and returns the real part of the resulting single half-wave integral. For the real Gaussian datum and zero initial velocity used here, $\operatorname{Re}(e^{it\abs{D}}u_0)=\cos(t\abs{D})u_0$, so this real projection is the exact wave solution. Mathematically, $\abs{\xi}$ is smooth and positively homogeneous of degree one on $\R^2\setminus\{0\}$. Training samples frequencies from the full square $[-40,40]^2$ and uses a negligible positive regularization of $\abs{\xi}$ at the origin. Accordingly, the homogeneous-FIO analysis applies away from the zero section. The model is trained over five seeds for 30,000 steps. The metric is the relative $L^{2}$ error at $t\in\{0,0.1,0.3,0.5\}$.

\paragraph{Results.} The relative $L^{2}$ is $2.1\times10^{-3}$ at $t=0$, $3.8\times10^{-3}\pm1.8\times10^{-3}$ at $t=0.1$, $1.6\times10^{-2}\pm5.7\times10^{-3}$ at $t=0.3$, and $2.2\times10^{-2}\pm6.0\times10^{-3}$ at $t=0.5$, with all five seeds below $0.03$ at $t=0.5$ and the worst seed at $0.0285$. The $t=0$ value is the initial-condition consistency check; that value is identical across seeds because the parameterization imposes $\Phi_\theta(0)=x\cdot\xi$ and $A_\theta(0)=1$ exactly.

\paragraph{Interpretation.} The same phase-space parameterization carries to four-dimensional phase space through one half-wave generator, a two-dimensional oscillatory reconstruction, and a real projection supplying the conjugate branch for this real zero-velocity datum. The error growth in time is comparable to that in the one-dimensional variable-coefficient case. This experiment demonstrates the parameterization with a full two-dimensional oscillatory quadrature; scalable quadrature is the main computational frontier in further dimensions.

\subsection{Efficiency and supervision}
\label{sec:exp-efficiency-supervision}

\subsubsection{Parameter count and wall-clock comparison}
\label{sec:exp-efficiency}

The convergence-regime advantage of Sec.~\ref{sec:exp-discont} occurs with a lower parameter count and a comparable or higher per-step cost, so that advantage is not attributable to model capacity or to cheaper steps.

\paragraph{Protocol.} Parameter counts follow from the architectures in Appendix~\ref{app:baseline-config}. Wall-clock time is measured on the 200,000-step runs using the common hardware described in Appendix~\ref{app:expdesign}.

\paragraph{Results.} Table~\ref{tab:params} reports the parameter counts and the measured per-step times. Single-branch MiNO uses 58,658 parameters, fewer than every baseline in the comparison; the two-branch wave configuration uses 117,316. On smooth advection, MiNO costs slightly more per step than the NTK-balanced PINN, and on the wave problem, two-branch MiNO costs slightly less than the causal PINN.

\begin{table}[pos=ht!]
\centering
\caption{Parameter counts and per-step wall-clock time on the 200,000-step runs. Every method listed is trained on residuals alone, with no labeled solution data, except the FNO, which is supervised on solutions under a separate epoch budget and is not timed here. The single-branch MiNO and NTK-balanced PINN rows are timed on smooth advection; the two-branch MiNO and causal PINN rows are timed on the wave problem.}
\label{tab:params}
\small
\begin{tabular}{@{}llrr@{}}
\toprule
Paradigm & Method & Parameters & Time per step (ms) \\
\midrule
\multirow{2}{*}{Propagator learning} & MiNO (single-branch)    & 58,658  & $2.5$ \\
 & MiNO (two-branch wave)  & 117,316 & $5.1$ \\
\addlinespace
\multirow{2}{*}{Field learning} & PINN (NTK-balanced)                   & 99,649  & $2.0$ \\
 & Causal PINN                           & 202,764 & $6.3$ \\
\addlinespace
Operator learning & FNO (supervised)                      & 329,281 & --- \\
\bottomrule
\end{tabular}
\end{table}

\paragraph{Interpretation.} Single-branch MiNO has the smallest parameter count among the compared models; the observed regime separation is therefore not explained by additional capacity. Its per-step cost on smooth advection reflects the phase-space eikonal and transport residuals and the automatic-differentiation derivatives those residuals require \citep{baydin2018autodiff}. The time-to-accuracy advantage therefore comes from needing far fewer steps to reach that scale (Sec.~\ref{sec:exp-discont}).

The residual-only objective in \eqref{eq:objective} excludes the oscillatory reconstruction, which is used only for prediction and periodic evaluation; the reported hardware-specific timings include that periodic evaluation under the stated collocation batching and derivative requirements.

\subsubsection{Supervised Fourier neural operator comparison}
\label{sec:exp-fno}

The two paradigms differ in what they consume as a training signal. This comparison records that asymmetry at the smooth-advection operating point of Sec.~\ref{sec:exp-smooth}; the comparison analyzes the runs reported there rather than introducing a new training run.

\paragraph{Protocol.} The supervised FNO baseline of \citet{li2021fno} is trained on solution data and compared against unsupervised MiNO on smooth advection under the same final-time metric (Sec.~\ref{sec:exp-smooth}). Appendix~\ref{app:baseline-config} records its architecture, and Table~\ref{tab:params} reports its parameter count.

\paragraph{Results.} The two final-time errors are those of Sec.~\ref{sec:exp-smooth}. MiNO is trained on residuals alone; the FNO is trained on 1024 labeled input--output pairs drawn from a band of sinusoids, so the localized Gaussian used for evaluation lies outside its training distribution (Appendix~\ref{app:baseline-config}).

\paragraph{Interpretation.} At this operating point, the residual-trained generator reaches the lower error while consuming no labeled solution data, placing the eikonal and transport residuals on the same footing as supervision for a translation symbol whose propagation geometry is simple.

\section{Discussion and conclusions}
\label{sec:discussion}

\subsection{Applicability and scope}

The three targets have complementary regimes. Field-learning methods are well suited to smooth, directly observable solutions and provide a flexible framework for inverse problems. Neural operators are effective when large trajectory datasets are available and fast amortized inference across a PDE family is the goal. Propagator learning is most relevant when the propagation geometry is simpler than the projected solution, as in transport, waves, sharp fronts, caustics, inverse propagation, and shared-symbol reuse across initial conditions. The exact-generator diagnostic demonstrates regular reconstruction through the free-particle focus studied here.

The three targets do not consume the same information, and the reported margins should be read accordingly. MiNO and the field baselines are both given the governing equation, but they are asked for different objects. On the constant-coefficient benchmarks, the exact phase increment is the linear function $h^{*}=-\xi$, constant in $t$ and $x$, whereas the corresponding field carries the transported front. The supervised FNO is given labeled pairs rather than the symbol. The comparisons therefore measure how much the choice of representational target changes the difficulty of the learning problem under a fixed governing equation, not the performance of methods supplied with identical information.

Datum regularity and proximity to a projection singularity provide natural qualitative coordinates for discussing the regime boundary (Figure~\ref{fig:failure}). The present evidence is scoped: The trained studies use linear problems in one and two space dimensions, a single phase chart, zero subprincipal symbols, and constant target amplitudes. The comparison covers two PINN variants and one supervised FNO at the stated budgets; the FNO test uses a localized Gaussian outside its sinusoidal training distribution. Nontrivial chart transitions, nonconstant learned amplitudes, broader PDE families, and additional matched baselines remain open evaluation directions. Appendix~\ref{app:limitations} gives the full limitations.

\begin{figure}[pos=t!]
\centering
\includegraphics[width=0.62\linewidth]{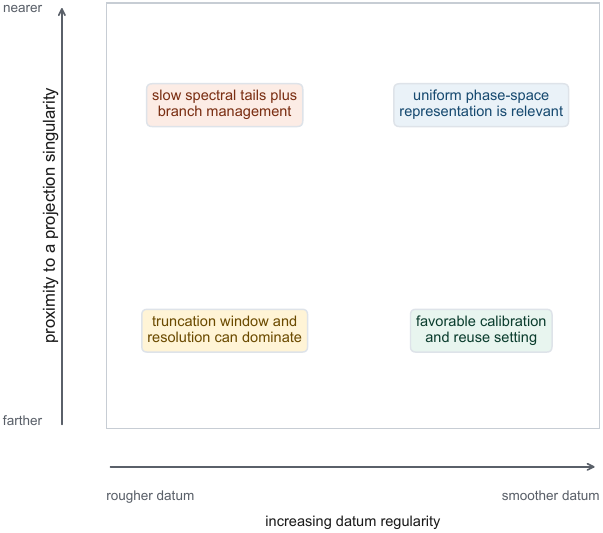}
\caption{Conceptual failure-mode coordinates based on datum regularity and proximity to a projection singularity. The four situations are qualitative prompts for method selection, not measured regimes or decision boundaries.}
\label{fig:failure}
\end{figure}

\subsection{Conclusions}
\label{sec:conclusion}

MiNO develops propagator learning as a third target alongside solution-field and solution-map learning. Its learned object is chartwise phase--amplitude data on the cotangent bundle: The eikonal equation constrains the phase, the transport equation evolves the amplitude, and an oscillatory quadrature returns the solution. Sharp fronts, discontinuities, and caustics are thereby represented through propagation geometry rather than only as features that a field surrogate must resolve pointwise.

The residual-to-error analysis developed in Sec.~\ref{sec:theory} controls the phase, the canonical graph, and the truncated propagator through residuals of the generator alone, while separating the frequency-truncation tail from that control. For discontinuous data the tail is sharp at rate $\Xi^{-1/2}$ (Proposition~\ref{prop:square-tail}), and its value provides a reference scale for interpreting the late-budget error on the chosen window.

Computationally, the target produces a distinct optimization regime in the wavefront-dominated benchmark. On discontinuous advection, MiNO reaches an error level consistent with the reconstruction-window tail while the NTK-balanced PINN remains near its initial error over the matched sweep, and at the reported smooth-advection operating point MiNO attains a lower mean error than the supervised FNO without using labeled trajectories. The remaining studies carry the same generator to bidirectional waves, to reuse across initial conditions, to inverse speed recovery, and to four-dimensional phase space, and the exact-generator diagnostic reconstructs through the focal time, with single-branch MiNO the smallest model in the comparison.

Propagator learning reframes neural PDE solving as learning the geometry that generates physical evolution rather than only the field that records that evolution. Where that geometry is simpler than the projected solution, the choice of target, and not the choice of optimizer, is what sets the difficulty of the learning problem.

\appendix
\counterwithin{equation}{section}
\counterwithin{table}{section}
\counterwithin{figure}{section}
\counterwithin{lstlisting}{section}

\section{Notation}
\label{app:notation}

Table~\ref{tab:notation} collects the symbols that recur throughout the paper, ordered as the construction is built: phase-space coordinates, the symbol and its canonical geometry, the learned generator and its Fourier-integral reconstruction, the training residuals, and the quantities that enter the residual-to-error estimates of Sec.~\ref{sec:theory}. Symbols confined to a single experiment or proof are defined at the point of occurrence.

\begin{table}[pos=t!]
\centering
\caption{List of notation: the symbols used repeatedly in the paper.}
\label{tab:notation}
\footnotesize
\begin{tabular}{@{}ll@{}}
\toprule
Symbol & Meaning \\
\midrule
\multicolumn{2}{@{}l}{\textbf{Phase space and data}} \\
$d,\Omega,[0,T]$ & Spatial dimension; domain $\R^d$, $\mathbb T^d$, or a chart; time interval \\
$\Tst$ & Cotangent bundle $\{(x,\xi)\}$; the phase space of the generator \\
$x,\xi$ & Base point; initial Fourier (fiber) label of the generator \\
$\zeta$ & Current covector $\zeta=\nabla_x\Phi_\theta(t,x,\xi)$ \\
$u_0,u(t,x)$ & Initial datum; solution of the Cauchy problem \\
$\wh v$ & Fourier transform of $v$ \\
$D_t,D_x$ & $-i\partial_t$ and $-i\nabla_x$; fix the eikonal sign convention \\
\addlinespace
\multicolumn{2}{@{}l}{\textbf{Symbol and canonical geometry}} \\
$p(x,\xi)$ & Principal symbol of the evolution operator \\
$p_{\mathrm{sub}}$ & Subprincipal symbol in the chosen (Weyl) quantization \\
$c(x)$ & Transport velocity of the scalar symbol $p=c(x)\xi$ \\
$H_p$ & Hamiltonian vector field $(\nabla_\zeta p,-\nabla_xp)$ \\
$\chi_t$ & Time-$t$ Hamiltonian (bicharacteristic) flow of $p$ \\
$X(s;t,x)$ & Characteristic flow of $c$, from time $t$ back to time $s$ \\
$\WF,\WF'$ & Wavefront set of a distribution; its twisted form for a kernel \\
$\mathcal C_t,\widetilde{\mathcal C}_t$ & Canonical relation of the exact propagator; of a learned phase \\
$\nu,m_\nu$ & Maslov index; its unit-modulus chart-transition factor \\
\addlinespace
\multicolumn{2}{@{}l}{\textbf{Learned generator and reconstruction}} \\
$\phi,A$ & Exact chart phase and amplitude \\
$\Phi_\theta,A_\theta$ & Learned phase and amplitude: the propagator generator \\
$h_\theta$ & Phase network; $\Phi_\theta=x\cdot\xi+th_\theta$ \\
$\alpha_\theta,B_\theta$ & Log-amplitude network; $B_\theta=\exp(t\alpha_\theta)$ \\
$F_{\phi,A}$ & Fourier-integral reconstruction operator \\
$R,K,\Xi$ & Bounded spatial output set; compact frequency window and its half-width \\
$P_K$ & Fourier projection onto the window $K$ \\
$\widetilde F^K,S_c^K,S_c$ & Learned truncated, exact truncated, and full transport propagators \\
\addlinespace
\multicolumn{2}{@{}l}{\textbf{Learning problem and residuals}} \\
$\theta$ & Trainable network parameters \\
$\mathcal R_{\mathrm{eik}}$ & Eikonal residual $\partial_t\Phi_\theta+p(x,\nabla_x\Phi_\theta)$ \\
$\mathcal R_{\mathrm{tr}}$ & Chart transport residual of the amplitude \\
$v_\theta$ & Group velocity $\nabla_\zeta p$ at the learned covector \\
$\mathcal J(\theta)$ & Propagator-learning objective \\
$\mu$ & Collocation sampling measure on $[0,T]\times\Tst$ \\
$u_\theta,\mathcal G_\theta$ & Learned field of a PINN; learned solution map of a neural operator \\
\addlinespace
\multicolumn{2}{@{}l}{\textbf{Residual-to-error estimates}} \\
$r,r_\phi,r_A$ & Residuals of an approximate phase; of phase and of amplitude \\
$e$ & Phase error $\widetilde\phi-\phi$ \\
$b,Y_b$ & Averaged velocity in \eqref{eq:averaged-velocity} and its terminal-value flow \\
$R_s,\mathcal Q_t$ & Output set transported to time $s$; the space--time tube swept by that set \\
$\kappa,\kappa_b$ & Suprema of the divergence of the transport velocity \\
$\mathfrak E_t,\mathfrak P_t$ & Residual functionals bounding amplitude and phase error \\
$\Gamma_\kappa(t)$ & Growth factor $((e^{\kappa t}-1)/\kappa)^{1/2}$ \\
$L$ & Lipschitz constant of $H_p$ \\
$\Delta_{\mathrm{cont}}$ & Continuous-residual bound on the truncated propagator \\
\bottomrule
\end{tabular}
\end{table}

\section{Related work}
\label{app:related}

This appendix positions MiNO through the three-target taxonomy of Sec.~\ref{sec:targets}. The comparison emphasizes what each method learns, how physics enters, and whether propagation geometry is represented explicitly.

\subsection{Solution-field learning and its remedies}
PINNs parameterize a field $u_{\theta}(t,x)$ and minimize strong-form PDE, initial-condition, and boundary residuals at collocation points \citep{raissi2019pinn,sirignano2018dgm,e2018deepritz}, with derivatives computed by automatic differentiation \citep{baydin2018autodiff}; the broader framework is surveyed by \citet{karniadakis2021piml}. This target is effective for smooth fields and remains a flexible mesh-free framework for inverse problems, but a transported discontinuity must still be represented as a nonsmooth function on the base manifold.

Two families of remedies address that difficulty without changing the learned object. Weak-form training replaces an undefined shock residual by a Kru\v zkov entropy formulation \citep{kruzkov1970first,deryck2024wpinn}, and Lagrangian PINNs move collocation points along characteristics, producing a frame in which the convection-dominated solution manifold has a more rapidly decaying $n$-width \citep{mojgani2023lpinn}. Causal weighting, gradient or NTK loss balancing, and modified MLPs with Fourier features instead improve the training dynamics \citep{wang2021gradient,wang2022causal,wang2022ntk, wang2023expert,krishnapriyan2021failure}. The Lagrangian formulation is the closest field-learning analog of MiNO, but that formulation transports coordinates while retaining a field target; MiNO learns the phase--amplitude generator of the flow on $\Tst$. The experiments compare against strong-residual NTK-balanced and causal PINNs (Sec.~\ref{sec:exp-discont}).

\subsection{Solution-operator learning}
The FNO \citep{li2021fno}, DeepONet \citep{lu2021deeponet}, and the general neural-operator framework \citep{kovachki2023neural} learn $\mathcal G_\theta:u_0\mapsto u(t,\cdot)$ between function spaces. DeepONet uses a branch--trunk factorization rooted in universal approximation for nonlinear operators \citep{chen1995universal}; the FNO alternates truncated spectral convolutions with pointwise operations. Both have approximation theory with explicit error bounds \citep{kovachki2021universal,lanthaler2022deeponet}. For advection-dominated problems with discontinuous solutions, however, linear reconstruction maps obey lower bounds that nonlinear reconstruction can avoid \citep{lanthaler2023nonlinear}. Physics-informed operator methods replace labeled trajectories by PDE residuals while preserving the same solution-map target \citep{wang2021pideeponet,li2024pino}.

Standard neural operators encode propagation geometry implicitly through the learned input--output map and its training distribution. MiNO parameterizes that geometry explicitly through a generator from which a solution map is synthesized by \eqref{eq:fio}. Its reuse across initial conditions follows directly from the linearity of that representation in $\wh u_0$, whereas operator networks learn amortized maps from pairs (Sec.~\ref{sec:exp-multiic}). The supervised FNO comparison in Sec.~\ref{sec:exp-fno} illustrates this separation at one smooth-advection operating point.

\subsection{Learned eikonal and analytic phase-space methods}
EikoNet, PINNeik, and related tomography methods learn the eikonal solution itself as a base-manifold traveltime field \citep{smith2021eikonet,binwaheed2021pinneik,chen2022eikonal}. These eikonal solvers are well suited to single-valued first-arrival computation, including the possibly nonsmooth viscosity solution \citep{crandall1983viscosity}. Classical fast-marching and sweeping methods address the first-arrival problem on a grid \citep{sethian1996fastmarching,zhao2005sweeping}; Eulerian phase-space, level-set, and ray-tube methods lift or decompose the multivalued arrival structure \citep{osher2002geometric,symes2003slowness,engquist2003high, runborg2007highfreq,benamou1996bigray}. Architectures that encode a Hamilton--Jacobi representation formula target the equation's solution \citep{darbon2020overcoming}; MiNO targets the phase--amplitude data of a PDE propagator and accommodates its branch structure.

The same lift underlies analytic high-frequency solvers. Maslov representations, Gaussian-beam and frozen-Gaussian superpositions, and curvelet representations prescribe phase-space structure that remains useful at focal points \citep{maslov1981semiclassical,ralston1982gaussian,liu2013gaussianbeam, lu2011frozen,candes2005curvelet}. Fast and butterfly FIO algorithms then apply a known oscillatory kernel through hierarchical low-rank factorization \citep{candes2007fast,candes2009butterfly}. These methods provide the geometric template and the efficient deployment engines for that template.

These methods also fix what a learned generator adds, since the classical solvers already propagate eikonal and transport data without any network. Their phase-space data are produced by a numerical construction, a ray or beam integration performed on a mesh or a beam family for each configuration; that construction is not differentiable with respect to the symbol that generated those data. MiNO produces the same data as a mesh-free closed-form map on $[0,T]\times\Tst$, evaluable at any point without reintegration and differentiable in its inputs and in the parameters of the symbol. The consequences are the ones exercised here: One trained generator serves new initial conditions through \eqref{eq:fio} alone (Sec.~\ref{sec:exp-multiic}), and a bank of generators enters a gradient-based parameter recovery as a structured forward model (Sec.~\ref{sec:exp-inverse}). What is learned is the representation, not the propagation principle, which stays the classical one.

\subsection{Learned symbols and microlocal neural architectures}
Phase-space structure is an established learning target for finite-dimensional dynamics, where the learned object is a Hamiltonian or a symplectic map \citep{greydanus2019hnn,jin2020sympnets}. For PDE propagators, the learned object has remained a field on the base manifold or a map between function spaces.

A close operator-learning neighbor parameterizes a pseudo-differential symbol $\sigma_\theta(x,\xi)$, generalizing the FNO while retaining the solution operator as its target \citep{shin2024pdno}. Its kernel has the fixed identity phase $(x-y)\cdot\xi$ and hence the diagonal canonical relation, which by Remark~\ref{rem:pseudo} bars that family from moving a wavefront set for any learned symbol. MiNO learns instead the admissible phase of a homogeneous hyperbolic FIO, whose canonical relation is locally the graph of $\chi_t$ (Proposition~\ref{prop:causal}); this phase carries the propagation geometry alongside the frequency-dependent amplitude.

Microlocal structure also appears in learned imaging. An FIO-inspired kernel can replace translation equivariance in variable-background wave imaging \citep{kothari2020geometry}, while wavefront-set extraction and completion can use the known canonical relation of the Radon transform to guide limited-angle reconstruction \citep{andradeloarca2022microlocal}. In the first case the trained object is an imaging map with an FIO prior; in the second the trained object is a data-domain singularity set used with a known forward operator. MiNO applies related microlocal structure to the eikonal- and transport-constrained generator of a PDE propagator.

\subsection{Classical microlocal theory and the present contribution}
The structure represented by MiNO is classical: geometric optics for oscillatory hyperbolic Cauchy data \citep{lax1957asymptotic}, the FIO representation and its eikonal--transport hierarchy, canonical relations, and Maslov transitions \citep{hormander1971fio,duistermaat1972fio,duistermaat1996fio, hormander1985vol3,hormander1985vol4,arnold1967maslov, maslov1981semiclassical,zworski2012semiclassical}, together with uniform oscillatory analysis at degenerate projections \citep{duistermaat1974oscillatory}.

The contribution here is to make the microlocal phase--amplitude generator a primary neural PDE-learning target and to connect its residual training to the generated object. The analysis gives compact-chart Hamilton--Jacobi phase and reconstruction stability (Theorem~\ref{thm:hj-phase-stability} and Corollary~\ref{cor:hj-reconstruction}), control of the canonical graph by differentiated eikonal residuals (Theorem~\ref{thm:canonical-residual} and Corollary~\ref{cor:canonical-uniform}), operator-norm control for scalar transport (Theorem~\ref{thm:transport-residual} and Corollary~\ref{cor:transport-residual-l2}), and a full-solution separation of generator error from the sharp discontinuity tail (Corollary~\ref{cor:continuous-full} and Proposition~\ref{prop:square-tail}). Together with the architecture and experiments, these results distinguish MiNO from learning a field, a solution map, an FIO prior for imaging, or a wavefront set under a known measurement operator.

\FloatBarrier
\section{Transport residuals for scalar advection and quadratic multipliers}
\label{app:transport}
\noindent Here $\xi$ is the initial Fourier (fiber) label of the generating function, while the current covector is $\zeta\coloneqq\nabla_x\Phi(t,x,\xi)$. Thus $A(t,x,\xi)$ denotes a chart amplitude pulled back to this fixed-label parameterization. A free phase-space amplitude evolves under the full Hamiltonian derivative $H_p\coloneqq\nabla_\zeta p\cdot\nabla_x-\nabla_xp\cdot\nabla_\zeta$, whereas the pulled-back chart amplitude follows the chart transport derived from its generating function. The ansatz below is for one normalized dynamical branch; prescribed localization cutoffs and any explicit external multi-branch weights are suppressed. We use
\begin{equation}
\label{eq:fioansatz}
u(t,x) = (2\pi)^{-d/2}\int_{\R^{d}} e^{i\Phi(t,x,\xi)}A(t,x,\xi)\wh u_0(\xi)\dd\xi , \qquad \Phi(0,x,\xi)=x\cdot\xi,\quad A(0,x,\xi)=1 .
\end{equation}

\subsection{Class 1: first-order transport $p(x,\xi)=c(x)\xi$, $d=1$}
The equation is $\partial_{t}u+c(x)\partial_{x}u=0$. Substituting \eqref{eq:fioansatz} and differentiating under the integral gives
\begin{align*}
\partial_{t}u+c\partial_{x}u =(2\pi)^{-1/2}\int e^{i\Phi} [i(\partial_{t}\Phi+c\partial_{x}\Phi)A +(\partial_{t}A+c\partial_{x}A)]\wh u_0\dd\xi .
\end{align*}
A sufficient exact split is to set the coefficient multiplying $iA$ and the remaining amplitude bracket separately to zero. With a semiclassical parameter inserted, these two conditions are the leading and next equations. The first condition is the eikonal equation $\partial_{t}\Phi+c(x)\partial_{x}\Phi=0$ with $\Phi(0,x,\xi)=x\xi$, solved by the characteristics $\dot{x}=c(x)$. Along those curves, $\Phi$ is constant, while its covector satisfies $\mathrm{d}(\partial_x\Phi)/\mathrm{d}t=-c'(x)\partial_x\Phi$. The remaining order is exactly the scalar advective equation $\partial_{t}A+c(x)\partial_{x}A=0$. If $X(s;t,x)$ is the characteristic flow, then
\begin{align*}
\Phi(t,x,\xi)=X(0;t,x)\xi,\qquad A\equiv1
\end{align*}
gives the scalar pullback $u(t,x)=u_0(X(0;t,x))$. In particular, for $c(x)=x$, $\Phi=e^{-t}x\xi$ and hence $\Delta_x\Phi=0$ while $c'(x)=1$, demonstrating that $\Delta_x\Phi$ and $c'(x)$ are independent quantities.

The alternative half-density equation $\partial_tu+c(x)\partial_xu+c'(x)u/2=0$ has amplitude residual $\partial_tA+c\partial_xA+c'A/2=0$ in the zero-subprincipal Weyl convention and represents a different PDE. The scalar experiments and residual construction in this paper target the former normalization and contain no $c'A/2$ term.

\subsection{Class 2: quadratic Fourier multiplier $p(\xi)$}
The symbol is $x$-independent, as for free Schr\"odinger $p(\xi)=\abs{\xi}^{2}/2$. The exact propagator is the Fourier multiplier $e^{-itp(\xi)}$, so \eqref{eq:fioansatz} is exact with
\begin{align*}
\Phi(t,x,\xi)=x\cdot\xi-tp(\xi),\qquad A\equiv 1 .
\end{align*}
Then $\nabla_{x}\Phi=\xi$ and the eikonal residual is $\partial_{t}\Phi+p(\nabla_{x}\Phi)=-p(\xi)+p(\xi)=0$ exactly; $\nabla_x^2\Phi=0$ and all derivatives of $A$ vanish. Hence both the leading transport equation and the exact finite-wavelength amplitude equation are satisfied.

For a general learned $x$-dependent phase and amplitude and $p(\zeta)=\zeta^TH\zeta/2+\ell\cdot\zeta+c_0$ with $H=H^T$, the leading half-density transport operator is
\begin{align*}
\partial_tA+(H\nabla_x\Phi+\ell)\cdot\nabla_xA +\tfrac12\operatorname{tr}(H\nabla_x^2\Phi)A +ip_{\mathrm{sub}}(x,\nabla_x\Phi)A.
\end{align*}
The quadratic multiplier considered here has $p_{\mathrm{sub}}=0$. For the Weyl quantization of the quadratic symbol above with no additional lower-order term, the exact unit-scale Schr\"odinger amplitude equation under the convention of \eqref{eq:cauchy} sets the leading transport operator equal to $i\operatorname{tr}(H\nabla_x^2A)/2$; in semiclassical scaling, the same correction is $O(\varepsilon)$. A nonzero lower-order Weyl symbol has its own unit-scale action on the amplitude and is represented by $i p_{\mathrm{sub}}A$ only at leading WKB order. For an arbitrary $x$-dependent learned amplitude, the transport residual above is therefore the leading WKB equation; exact finite-frequency evolution also contains the corresponding higher-order terms.

\subsection{Extension to other symbols}
The invariant leading half-density geometric term is $\operatorname{div}_x(\nabla_\zeta p(x,\nabla_x\Phi))A/2$; that term reduces to $\Delta_x\Phi A/2$ for the identity quadratic multiplier. The full leading transport equation also contains $i p_{\mathrm{sub}}A$ when the chosen quantization has nonzero subprincipal symbol. In the two direct reductions above, Class~1 follows by direct substitution into the scalar PDE, and Class~2 has $\nabla_xp=0$ and $A\equiv1$. For each additional chart, including an $x$-dependent quadratic lens and the two-dimensional half-wave case, the transport law is determined by its symbol, normalization, and quantization.

\FloatBarrier
\section{Proofs}
\label{app:proofdetails}
This appendix collects the proofs of the statements of Secs.~\ref{sec:targets} and~\ref{sec:theory}. Each proof uses the notation and the standing hypotheses of the setting that precedes its statement in the main text.

\subsection{Bicharacteristic kernel relation}

\begin{proof}[Proof of Proposition~\ref{prop:causal}]
At the fixed time in the statement, write the kernel as the oscillatory integral
\begin{align*}
\mathcal K_t(x,y) = (2\pi)^{-d}\int_{\R^{d}} e^{i(\phi(t,x,\xi)-y\cdot\xi)}A(t,x,\xi)\dd\xi .
\end{align*}
By the theory of oscillatory integrals \citep[Ch.~VIII]{hormander1983vol1}, the wavefront set of a distribution defined by such an integral is contained in the set of $(x,y;\xi_{x},\xi_{y})$ at which the phase $\Psi(x,y,\xi) \coloneqq \phi(t,x,\xi) - y\cdot\xi$ is stationary in the fiber variable $\xi$ and the differential in $(x,y)$ is the corresponding covector. Stationarity in $\xi$ reads $\nabla_{\xi}\phi(t,x,\xi)=y$. The unprimed kernel covectors are $(x,y;\nabla_x\phi,-\xi)$. By definition, the primed kernel relation reverses the input sign and therefore contains $(x,\nabla_x\phi;y,\xi)$. The set $\mathcal C_t\coloneqq\{(x,\nabla_x\phi;\nabla_{\xi}\phi,\xi)\}$ is, by the nonsingularity of $\partial^{2}_{x\xi}\phi$, a canonical graph, the graph of a homogeneous canonical transformation. Because $\phi$ solves the eikonal equation in \eqref{eq:eikonal} with Cauchy datum $\phi(0,x,\xi)=x\cdot\xi$, the method of characteristics for Hamilton--Jacobi equations \citep[Ch.~3]{evans2010pde} identifies the characteristic curves of \eqref{eq:eikonal} with the bicharacteristic curves of $p$, the integral curves of the Hamiltonian vector field $H_{p}$. Hence $\mathcal C_t$ is the graph of the time-$t$ bicharacteristic flow $\chi_{t}$. The wavefront containment gives $\WF'(\mathcal K_t)\subseteq\operatorname{graph}\chi_{t}$.
\end{proof}

\subsection{Hamilton--Jacobi phase and reconstruction stability}

\begin{proof}[Proof of Theorem~\ref{thm:hj-phase-stability}]
The fundamental theorem of calculus in the covector variable gives the exact pointwise identity
\begin{align*}
p(z,\nabla_z\widetilde\phi)-p(z,\nabla_z\phi) =b(s,z,\xi)\cdot\nabla_z e.
\end{align*}
Subtracting the exact Hamilton--Jacobi equation for $\phi$ from the definition of $r$ therefore yields the linear equation
\begin{equation}
\label{eq:phase-error-transport}
\partial_s e+b\cdot\nabla_z e=r.
\end{equation}
Along the characteristic in \eqref{eq:averaged-flow}, integration of \eqref{eq:phase-error-transport} gives
\begin{equation}
\label{eq:phase-error-char}
e(t,x,\xi)=e(0,Y_b(0;t,x,\xi),\xi) +\int_0^t r(s,Y_b(s;t,x,\xi),\xi)\dd s.
\end{equation}
For fixed $s$ and $\xi$, put $z\coloneqq Y_b(s;t,x,\xi)$. The inverse map is $x=Y_b(t;s,z,\xi)$, and Liouville's formula gives
\begin{align*}
\left|\det\frac{\partial Y_b(t;s,z,\xi)}{\partial z}\right| =\exp\left(\int_s^t \operatorname{div}_z b(\tau,Y_b(\tau;s,z,\xi),\xi)\dd\tau\right) \le e^{\kappa_b(t-s)}.
\end{align*}
Consequently, for every measurable $f$ on $R_s^b$,
\begin{align*}
\norm{f(Y_b(s;t,\cdot,\cdot),\cdot)}_{L^2(R\times K)} \le e^{\kappa_b(t-s)/2}\norm{f}_{L^2(R_s^b)}.
\end{align*}
Applying this inequality to the initial term and each integrand in \eqref{eq:phase-error-char}, followed by Minkowski's integral inequality, proves \eqrefs{eq:hj-phase-bound}{eq:phase-residual-functional}.
\end{proof}

\begin{proof}[Proof of Corollary~\ref{cor:hj-reconstruction}]
After the unitary Fourier transform, the localized difference of the two truncated reconstructions is an integral operator from $L^2(K_\xi)$ to $L^2(R_x)$ with kernel
\begin{align*}
(2\pi)^{-d/2}\left( e^{i\widetilde\phi}\widetilde A-e^{i\phi}A\right).
\end{align*}
Because both phases are real,
\begin{align*}
\abs{e^{i\widetilde\phi}\widetilde A-e^{i\phi}A} \le \abs{\widetilde A-A} +\abs{A}\cdot\abs{e^{i\widetilde\phi}-e^{i\phi}} \le \abs{\widetilde A-A}+\abs{A}\cdot\abs e.
\end{align*}
The operator norm is bounded by the Hilbert--Schmidt norm of this kernel. The triangle inequality in $L^2(R\times K)$ and \eqref{eq:hj-phase-bound} prove \eqref{eq:hj-reconstruction-bound}.
\end{proof}

\begin{proof}[Proof of Corollary~\ref{cor:hj-exact-data}]
Cauchy--Schwarz in time gives
\begin{align*}
\int_0^t e^{\kappa_b(t-s)/2}\norm{r(s)}_{L^2(R_s^b)}\dd s \le\left(\int_0^t e^{\kappa_b(t-s)}\dd s\right)^{1/2} \norm{r}_{L^2(\mathcal Q_b(t))},
\end{align*}
and the first factor is $\Gamma_{\kappa_b}(t)$. This estimate proves \eqref{eq:hj-exact-data}.
\end{proof}

\subsection{Canonical-geometry estimates}

\begin{proof}[Proof of Lemma~\ref{lem:canonical-characteristic}]
Along the curve in \eqref{eq:approx-char}, differentiate $\widetilde\zeta(s)=\nabla_x\widetilde\phi (s,\widetilde x(s),\xi)$. The chain rule and the $x$-gradient of the residual give
\begin{align*}
\dot{\widetilde\zeta} &=\partial_t\nabla_x\widetilde\phi +\nabla_x^2\widetilde\phi \nabla_\zeta p(\widetilde x,\widetilde\zeta),\\
\nabla_xr &=\partial_t\nabla_x\widetilde\phi +\nabla_xp(\widetilde x,\widetilde\zeta) +\nabla_x^2\widetilde\phi \nabla_\zeta p(\widetilde x,\widetilde\zeta).
\end{align*}
Therefore, with $\widetilde w\coloneqq(\widetilde x,\widetilde\zeta)$,
\begin{equation}
\label{eq:forced-ham}
\dot{\widetilde w}(s)=H_p(\widetilde w(s)) +(0,\nabla_xr(s,\widetilde x(s),\xi)).
\end{equation}
The exact initial phase gives $\widetilde\zeta(0)=\xi$. If $w(s)=\chi_s(y_0,\xi)$, subtracting the exact Hamiltonian equation from \eqref{eq:forced-ham}, using the $L$-Lipschitz property, and applying Gronwall's inequality proves \eqref{eq:forced-flow-bound}.

For the footpoint estimate, define $q(s)\coloneqq\nabla_\xi\widetilde\phi (s,\widetilde x(s),\xi)$. Differentiating the residual with respect to the fixed label $\xi$ and using \eqref{eq:approx-char} gives the exact identity
\begin{align*}
\dot q(s)=\nabla_\xi r(s,\widetilde x(s),\xi).
\end{align*}
Since the derivative defining $q$ is taken at fixed spatial argument, $q(0)=[\nabla_\xi\widetilde\phi(0,x,\xi)]_{x=y_0}=y_0$, and $q(t)=\widetilde y$. Integration proves \eqref{eq:footpoint-bound}.
\end{proof}

\begin{proof}[Proof of Theorem~\ref{thm:canonical-residual}]
The $L$-Lipschitz assumption and Gronwall's inequality imply continuous dependence of the exact Hamiltonian flow on its initial point:
\begin{align*}
\abs{\chi_t(\widetilde y,\xi)-\chi_t(y_0,\xi)} \le e^{Lt}\abs{\widetilde y-y_0}.
\end{align*}
The triangle inequality and the two estimates in Lemma~\ref{lem:canonical-characteristic} prove \eqref{eq:canonical-defect}. The mixed-Hessian assumption identifies the relation generated by the phase as a local canonical graph; that assumption is not needed for the differential inequalities in the lemma.
\end{proof}

\begin{proof}[Proof of Corollary~\ref{cor:canonical-uniform}]
Apply \eqref{eq:canonical-defect} to each $(x,\xi)\in\mathcal D_t$, take the uniform derivative bounds, and evaluate the exponential integral.
\end{proof}

\subsection{Scalar-transport stability and frequency-tail estimates}

\begin{proof}[Proof of Lemma~\ref{lem:tube-error}]
Integration along the characteristic ending at $x$ at time $t$ gives
\begin{equation}
\label{eq:char-error-rep}
e(t,x,\xi)=e_0(X(0;t,x),\xi) +\int_0^t r(s,X(s;t,x),\xi)\dd s.
\end{equation}
For $z\coloneqq X(s;t,x)$, the inverse change of variables is $x=X(t;s,z)$. Liouville's formula for the flow Jacobian gives
\begin{align*}
\left|\det\frac{\partial X(t;s,z)}{\partial z}\right| =\exp\left(\int_s^t \operatorname{div}c(X(\tau;s,z))\dd\tau\right) \le e^{\kappa(t-s)}.
\end{align*}
Consequently, for every measurable $f$ on $R_s\times K$,
\begin{equation}
\label{eq:flow-pullback-l2}
\norm{f(X(s;t,\cdot),\cdot)}_{L^2(R\times K)} \le e^{\kappa(t-s)/2}\norm{f}_{L^2(R_s\times K)}.
\end{equation}
Apply \eqref{eq:flow-pullback-l2} to the initial term and to each integrand in \eqref{eq:char-error-rep}, and then use Minkowski's integral inequality. These two steps prove \eqref{eq:generator-error-from-residual}.
\end{proof}

\begin{proof}[Proof of Theorem~\ref{thm:transport-residual}]
The backward-flow identity implies $\partial_s\phi_c+c\cdot\nabla_z\phi_c=0$ and $\phi_c(0,z,\xi)=z\cdot\xi$. Hence the phase error $e_\phi\coloneqq\widetilde\phi-\phi_c$ and amplitude error $e_A\coloneqq\widetilde A-1$ satisfy, respectively,
\begin{align*}
(\partial_s+c\cdot\nabla_z)e_\phi=r_\phi, \qquad (\partial_s+c\cdot\nabla_z)e_A=r_A.
\end{align*}
Lemma~\ref{lem:tube-error} bounds both errors by their corresponding functionals in \eqref{eq:residual-functional}. After the unitary Fourier transform, the difference of the two truncated operators has kernel
\begin{align*}
(2\pi)^{-d/2}\mathbf 1_R(x)\mathbf 1_K(\xi) \left(e^{i\widetilde\phi(t,x,\xi)}\widetilde A(t,x,\xi) -e^{i\phi_c(t,x,\xi)}\right).
\end{align*}
Both phases are real, so
\begin{align*}
\abs{e^{i\widetilde\phi}\widetilde A-e^{i\phi_c}} \le \abs{\widetilde A-1} +\abs{e^{i\widetilde\phi}-e^{i\phi_c}} \le \abs{e_A}+\abs{e_\phi}.
\end{align*}
The operator norm is bounded by the Hilbert--Schmidt norm of this kernel. Using the triangle inequality in $L^2(R\times K)$ and then \eqref{eq:generator-error-from-residual} for the two errors proves \eqref{eq:residual-operator-bound}.
\end{proof}

\begin{proof}[Proof of Corollary~\ref{cor:transport-residual-l2}]
With the initial-error terms absent from \eqref{eq:residual-functional}, Cauchy--Schwarz in time gives
\begin{align*}
\int_0^t e^{\kappa(t-s)/2}\norm{r(s)}_{L^2(R_s\times K)}\dd s \le\left(\int_0^t e^{\kappa(t-s)}\dd s\right)^{1/2} \norm{r}_{L^2(\mathcal Q_t)},
\end{align*}
and the first factor is exactly $\Gamma_\kappa(t)$. Apply this estimate to both residuals in \eqref{eq:residual-operator-bound}.
\end{proof}

\begin{proof}[Proof of Corollary~\ref{cor:continuous-full}]
The exact truncated operator satisfies $S_c^K(t)=S_c(t)P_K$, because Fourier inversion gives
\begin{align*}
(S_c(t)P_Ku_0)(x)=(2\pi)^{-d/2}\int_K e^{iX(0;t,x)\cdot\xi}\wh u_0(\xi)\dd\xi.
\end{align*}
Consequently,
\begin{align*}
\widetilde F^K(t)u_0-S_c(t)u_0 =(\widetilde F^K(t)-S_c^K(t))u_0 -S_c(t)(I-P_K)u_0.
\end{align*}
For the second term, the change of variables $z\coloneqq X(0;t,x)$ and Liouville's formula, exactly as in \eqref{eq:flow-pullback-l2}, give
\begin{align*}
\norm{S_c(t)(I-P_K)u_0}_{L^2(R)} \le e^{\kappa t/2}\norm{(I-P_K)u_0}_{L^2(R_0)} \le e^{\kappa t/2}\norm{(I-P_K)u_0}_{L^2(\R^d)}.
\end{align*}
The triangle inequality, Corollary~\ref{cor:transport-residual-l2}, and \eqref{eq:continuous-residual-delta} prove \eqref{eq:full-solution-bound}.
\end{proof}

\begin{proof}[Proof of Proposition~\ref{prop:square-tail}]
Under the Fourier convention of Sec.~\ref{sec:targets-gen},
\begin{align*}
\wh u_0(\xi)=\sqrt{\frac2\pi}\frac{\sin(a\xi)}{\xi}, \qquad \norm{u_0}_{L^2(\R)}^2=2a.
\end{align*}
Plancherel's theorem yields
\begin{equation}
\label{eq:square-tail-integral}
\norm{(I-P_K)u_0}_{L^2(\R)}^2 =\frac4\pi\int_\Xi^\infty\frac{\sin^2(a\xi)}{\xi^2}\dd\xi.
\end{equation}
The inequality $\sin^2(a\xi)\le1$ proves \eqref{eq:square-tail-bound}. A change of variables and integration by parts give
\begin{align*}
\int_\Xi^\infty\frac{\sin^2(a\xi)}{\xi^2}\dd\xi =a\left[ \frac{\sin^2(a\Xi)}{a\Xi}+\frac{\pi}{2}-\operatorname{Si}(2a\Xi) \right].
\end{align*}
Together with \eqref{eq:square-tail-integral} and $\norm{u_0}_{L^2(\R)}^2=2a$, this evaluation proves \eqref{eq:square-tail-exact}. Moreover,
\begin{align*}
\int_\Xi^\infty\frac{\sin^2(a\xi)}{\xi^2}\dd\xi =\frac1{2\Xi} -\frac12\int_\Xi^\infty\frac{\cos(2a\xi)}{\xi^2}\dd\xi =\frac1{2\Xi}+O(\Xi^{-2}),
\end{align*}
where the last estimate follows from one integration by parts. Inserting this expansion into \eqref{eq:square-tail-integral}, dividing by $2a$, and taking the square root proves \eqref{eq:square-tail-asymptotic}.
\end{proof}

\FloatBarrier
\section{Experimental design and implementation details}
\label{app:experimental-appendix}

\subsection{Evaluation design}
\label{app:expdesign}

The evaluation spans propagation regimes, representational fidelity and reuse, inverse recovery, dimensional extension, efficiency, and supervision. Table~\ref{tab:evalmap} summarizes the objective and benchmark associated with each study.

\begin{table}[pos=t!]
\centering
\caption{Evaluation overview.}
\label{tab:evalmap}
\small
\begin{tabular*}{\linewidth}{@{\extracolsep{\fill}}lll@{}}
\toprule
Objective & Benchmark & Section \\
\midrule
\multicolumn{3}{@{}l}{\textbf{Propagation regimes}} \\
Closed-form propagation calibration & smooth advection & Sec.~\ref{sec:exp-smooth} \\
Sharp-transport regime comparison & discontinuous advection & Sec.~\ref{sec:exp-discont} \\
Two-branch propagation & wave equation & Sec.~\ref{sec:exp-wave} \\
Inhomogeneous eikonal construction & stationary-point transport & Sec.~\ref{sec:exp-varcoeff} \\
Cross-paradigm variable-coefficient comparison & sinusoidal-speed transport & Sec.~\ref{sec:exp-varcoeff-sin} \\
Exact-generator reconstruction through a focus & point-focus caustic & Sec.~\ref{sec:exp-caustic} \\
\addlinespace
\multicolumn{3}{@{}l}{\textbf{Representational content, reuse, and extension}} \\
Finite-resolution microlocal fidelity & wave-packet portrait & Sec.~\ref{sec:exp-spectral} \\
Propagator reuse across initial data & multi-initial-condition reuse & Sec.~\ref{sec:exp-multiic} \\
Structured inverse recovery & transport-speed recovery & Sec.~\ref{sec:exp-inverse} \\
Dimensional extension & two-dimensional propagation & Sec.~\ref{sec:exp-2d} \\
\addlinespace
\multicolumn{3}{@{}l}{\textbf{Efficiency and supervision}} \\
Capacity and per-step cost & parameter-count and wall-clock comparison & Sec.~\ref{sec:exp-efficiency} \\
Supervision requirement & smooth-advection MiNO--FNO comparison & Sec.~\ref{sec:exp-fno} \\
\bottomrule
\end{tabular*}
\end{table}

\paragraph{Common protocol.} Representation diagnostics use exact generators to isolate the reconstruction, whereas learning benchmarks use the learned-generator specialization documented in Appendix~\ref{app:experimental-residual}. Errors are evaluated either at the final time or over the full space--time grid, as specified for each result. The seed protocol and the per-seed aggregation conventions behind the quoted dispersions are collected in Appendix~\ref{app:expdetails}.

\paragraph{Baselines and computing environment.} The cross-paradigm comparisons use two field-learning baselines trained over five seeds for 200,000 steps: the NTK-balanced PINN, a modified MLP with Fourier features and NTK balancing in the form of \citet{wang2023expert}, and the causal PINN of \citet{wang2022causal}. The supervised FNO of \citet{li2021fno} is trained separately on labeled solution data under its own epoch budget (Appendix~\ref{app:baseline-config}). All runs use a single NVIDIA A100 GPU. MiNO and the NTK-balanced PINN are implemented in JAX \citep{bradbury2018jax}. Each experiment states the applicable subset of baselines, and Appendix~\ref{app:baseline-config} records their configurations.

\subsection{Experimental residual specialization}
\label{app:experimental-residual}
The experiments directly parameterize the full phase increment $h_\theta$, giving the phase $x\cdot\xi+t h_\theta$, and use the amplitude $1+t\gamma_\theta$. The default one-dimensional reconstruction is the composite trapezoidal rule. The point-focus diagnostic uses the Filon rule of Appendix~\ref{app:expdetails}. During amplitude training, the experimental residual is
\begin{align*}
\partial_tA+\nabla_\zeta p(x,\zeta)\cdot\nabla_xA -\nabla_xp(x,\zeta)\cdot\nabla_\xi A, \qquad \zeta\coloneqq\nabla_x\Phi_\theta(t,x,\xi).
\end{align*}
Both symbol derivatives are evaluated at the learned covector, matching Listing~\ref{lst:experimental-kernel}. This residual is the full Hamiltonian derivative of a free phase-space function, specialized to $p_{\mathrm{sub}}=0$. For general chart amplitudes, \eqref{eq:Rtr} uses the pulled-back chart derivative and the term $i p_{\mathrm{sub}}A$, so the two formulations are not identical away from their target. The trained experiments are restricted to the constant-amplitude subclass $A^*=1$ with $p_{\mathrm{sub}}=0$, for which both formulations have the same exact zero target.

\paragraph{JAX implementation of the experimental kernel.} Listing~\ref{lst:experimental-kernel} records the minimal computational core used by the trained experiments. Here \texttt{h\_apply} and \texttt{gamma\_apply} are parameter-closed Flax-network outputs, while \texttt{p} evaluates one selected symbol branch; the derivatives of \texttt{p} are partial derivatives with the other argument fixed. Every trained case uses a real amplitude and $p_{\mathrm{sub}}=0$, so no subprincipal term appears. The residuals are batched with \texttt{jax.vmap} and combined as in \eqref{eq:objective}.

\begin{lstlisting}[style=mino-python, float=tp, floatplacement=tp,
  caption={Minimal JAX kernel used in the trained MiNO experiments.},
  label={lst:experimental-kernel}]
import jax
import jax.numpy as jnp
# Exact phase and amplitude data at t = 0.
def make_generator(h_apply, gamma_apply):
    def phi(t, x, xi):
        h = jnp.asarray(h_apply(t, x, xi)).squeeze()
        return jnp.dot(x, xi) + t * h
    def amplitude(t, x, xi):
        gamma = jnp.asarray(gamma_apply(t, x, xi)).squeeze()
        return 1.0 + t * gamma
    return phi, amplitude
# Eikonal residual on the learned Lagrangian.
def eikonal_residual(phi, p, t, x, xi):
    phi_t = jax.grad(phi, argnums=0)(t, x, xi)
    zeta = jax.grad(phi, argnums=1)(t, x, xi)
    return phi_t + p(x, zeta)
# Full-Hamiltonian amplitude residual used in the experiments.
def experimental_transport_residual(
        phi, amplitude, p, t, x, xi):
    a_t = jax.grad(amplitude, argnums=0)(t, x, xi)
    a_x = jax.grad(amplitude, argnums=1)(t, x, xi)
    a_xi = jax.grad(amplitude, argnums=2)(t, x, xi)
    zeta = jax.grad(phi, argnums=1)(t, x, xi)
    p_x = jax.grad(p, argnums=0)(x, zeta)
    p_zeta = jax.grad(p, argnums=1)(x, zeta)
    return a_t + jnp.dot(p_zeta, a_x) - jnp.dot(p_x, a_xi)
\end{lstlisting}

\subsection{Networks, optimization, and reconstruction}
\label{app:expdetails}
\paragraph{Networks and parameterization.} The phase network is a modified MLP with four hidden layers of width 64, multiplicative skip connections in the form of \citet{wang2023expert}, and 32 sinusoidal input features at the fixed scales $0.25$ in $x$ and $0.1$ in $\xi$. The amplitude network is a modified MLP with three hidden layers of width 64 and 16 sinusoidal input features. The phase is parameterized as in \eqref{eq:phase}. The experimental amplitude $1+t\gamma_\theta$ enforces its initial value; the general positivity-preserving model $\exp(t\alpha_{\theta})$ and its unit-complex Maslov factor are defined in Sec.~\ref{sec:mino-param}. Appendix~\ref{app:experimental-residual} gives the experimental specialization.

\paragraph{Optimization and collocation.} Optimization uses Adam \citep{kingma2015adam} with a constant learning rate $10^{-3}$ and 4096 collocation points resampled each step from a product measure on time, space, and frequency. This fixed, nonadaptive sampling design follows the comparison setting of \citet{wu2023sampling}. When enabled by the individual run configuration, online NTK balancing updates the eikonal and transport weights every 100 steps.

\paragraph{Seeds and per-seed aggregation.} Unless stated otherwise, results for trained MiNO and for the baseline methods are reported over five seeds as the mean and sample standard deviation. The multi-initial-condition generator of Sec.~\ref{sec:exp-multiic} is trained once with a single seed, so the errors reported there are single-run values. The inverse study of Sec.~\ref{sec:exp-inverse} uses five noise seeds for each combination of true speed and noise level, giving fifteen configurations per noise level and sixty in total.

\paragraph{Reconstruction windows and domains.} Reconstruction uses the composite trapezoidal rule by default. The smooth-advection run uses 1024 nodes on $[-30,30]$, the wave run uses 512 nodes on $[-100,100]$, and the discontinuous-advection run uses 1024 nodes on $[-100,100]$. The two sinusoidal-speed runs of Sec.~\ref{sec:exp-varcoeff-sin} use 512 nodes on $[-30,30]$ for the Gaussian datum and 1024 nodes on $[-60,60]$ for the square wave. The stationary-point run of Sec.~\ref{sec:exp-varcoeff} uses 1024 nodes on $[-20,20]$, and the two-dimensional run of Sec.~\ref{sec:exp-2d} samples training frequencies from $[-40,40]^2$. The optional Filon rule is used in the exact-generator point-focus test, whose parameters are listed with that experiment. Because the windows and node counts differ across runs, every reported error includes its own quadrature contribution. Each run samples collocation points over the space--time domain of its own experiment, as stated with that experiment, and over the frequency window listed above for that run.

\paragraph{Filon-type panel formula.} The optional Filon rule of Sec.~\ref{sec:mino-recon} is the following piecewise-linear-phase rule. On each frequency panel $[\xi_n,\xi_{n+1}]$ of width $\Delta\xi$, with $s\coloneqq(\xi-\xi_n)/\Delta\xi\in[0,1]$, both the phase and the smooth factor $g=A_\theta\wh u_0$ are interpolated linearly,
\begin{align*}
\Phi(\xi)\approx \Phi_n+\beta_ns,\qquad g(\xi)\approx(1-s)g_n+s g_{n+1},\qquad \beta_n\coloneqq\Phi_{n+1}-\Phi_n ,
\end{align*}
and the panel contribution is evaluated as
\begin{align*}
\Delta\xi e^{i\Phi_n}[g_n(E_1(\beta_n)-E_2(\beta_n)) +g_{n+1}E_2(\beta_n)],
\end{align*}
where $E_1(\beta)\coloneqq\int_0^1e^{i\beta s}\dd s$ and $E_2(\beta)\coloneqq\int_0^1s e^{i\beta s}\dd s$, using their Taylor limits near $\beta=0$.

\subsection{Baseline configurations}
\label{app:baseline-config}
The PINN baseline is a modified-MLP architecture implemented for this study, with four hidden layers of width 128, 64 Fourier features, and NTK balancing in the form of \citet{wang2023expert}; the baseline has 99,649 parameters and is trained for 200,000 steps over five seeds. The causal baseline of \citet{wang2022causal} uses PirateNet with three layers of width 256, 202,764 parameters, and the same training budget. The operator baseline uses the FNO architecture of \citet{li2021fno}, with 32 modes, four layers of width 64, and 329,281 parameters, trained supervised on solution data. The FNO is trained over the same five seeds for 200 epochs with batch size 32 and learning rate $10^{-3}$ on 1024 supervised pairs $(u_{0},u(t_{\max},\cdot))$, with 256 further pairs held out. Each training datum is a normalized sum of one to three sinusoids $\sin(\pi f(x-\varphi))$ with $f$ drawn uniformly from $[0.25,1.5]$ and $\varphi$ from $[0,1]$, and its label is the corresponding translation on $[-3,3]$. The evaluation datum of Sec.~\ref{sec:exp-smooth} is the localized Gaussian, whose spectrum extends beyond that training band, so the reported FNO error is that of a trained operator applied to an initial condition outside its training distribution.

\FloatBarrier
\section{Limitations}
\label{app:limitations}

The evidence reported here comes from a scoped evaluation rather than an exhaustive one. The test suite consists of linear propagation problems in one and two space dimensions, each generated from the initial data, coefficients, and phase-space windows of Appendix~\ref{app:expdetails}; the comparison covers NTK-balanced and causal PINNs and a supervised FNO. That evaluation therefore does not span every benchmark dataset or every field-learning and operator-learning architecture in current use, and the reported margins are specific to the regimes and baselines examined. Wider coverage, in particular further PDE families and additional baselines under matched data, resolution, and hyperparameter budgets \citep{lu2022fair}, would place the propagator target on a broader empirical footing.

The trained configurations are also narrower than the construction of Sec.~\ref{sec:mino}. Every trained run stays on a single phase chart with $m_\nu=1$, so no chart crossing and no nontrivial Maslov transition is exercised; the two-branch wave configuration superposes two single-chart generators under explicit external weights rather than transitioning between charts. All trained runs take $p_{\mathrm{sub}}=0$ and the constant-amplitude specialization $A^{*}=1$ (Appendix~\ref{app:experimental-residual}), so the positivity-preserving amplitude model $\exp(t\alpha_{\theta})$ and general nonconstant chart amplitudes are defined but not trained here. The supervised comparison of Sec.~\ref{sec:exp-fno} is likewise at a single labeled-data budget.

\FloatBarrier

\printcredits

\section*{Declaration of competing interest}
The authors declare that they have no known competing financial interests or personal relationships that could have appeared to influence the work reported in this paper.

\section*{Data availability}
The MiNO implementation and the experiment records that reproduce the reported numbers are available from the corresponding author on request.

\section*{Funding}
This research received no specific grant from any funding agency in the public, commercial, or not-for-profit sectors.

\section*{Acknowledgments}
The authors thank the Axiom Research Group for sustained discussion of the framework.

\bibliographystyle{cas-model2-names}
\bibliography{mino}

\end{document}